\documentclass{article}

\PassOptionsToPackage{numbers, sort}{natbib}

\usepackage[preprint]{neurips_2022}

\usepackage[utf8]{inputenc} 
\usepackage[T1]{fontenc}    
\usepackage{hyperref}       
\usepackage{url}            
\usepackage{booktabs}       
\usepackage{nicefrac}       
\usepackage{microtype}      
\usepackage{xcolor}         

\usepackage{latex-macros}
\usepackage{mathtools}
\usepackage{empheq}
\usepackage{multirow}
\usepackage[capitalize]{cleveref}
\usepackage{amssymb,amsthm}
\usepackage[noend]{algpseudocode}
\usepackage{setspace}
\usepackage{appendix}
\usepackage{cancel}
\usepackage{enumitem}
\usepackage{algorithm,algpseudocode}
\usepackage[labelfont=bf]{caption}
\definecolor{lblue}{HTML}{908cc0}
\definecolor{mblue}{HTML}{519cc8}
\definecolor{hblue}{HTML}{1d5996}
\definecolor{lred}{HTML}{cb5501}
\definecolor{mred}{HTML}{f1885b}
\definecolor{hred}{HTML}{b3001e}
\definecolor{ttred}{HTML}{ca3542}

\newcommand{\bigopnorm}[1]{\bigl\| {#1} \bigr\|_{2}}

\newcommand{\biggopnorm}[1]{\biggl\| {#1} \biggr\|_{2}}

\newcommand{\Bigopnorm}[1]{\Bigl\| {#1} \Bigr\|_{2}}
\newcommand{\lub}{\lambda_{\mathsf{ub}}}
\newcommand{\llb}{\lambda_{\mathsf{lb}}}

\newtheorem{theorem}{Theorem}
\newtheorem{lemma}{Lemma}
\newtheorem{corollary}{Corollary}
\newtheorem{proposition}{Proposition}

\theoremstyle{definition}
\newtheorem{assumption}{Assumption}

\Crefname{assumption}{Assumption}{Assumptions}
\Crefname{theorem}{Theorem}{Theorems}
\Crefname{definition}{Definition}{Definitions}

\newtheorem{remark}{Remark}
\Crefname{remark}{Remark}{Remarks}

\newcommand{\est}{A}
\newcommand{\sol}{A}
\newcommand{\mtrv}{\widehat{V}}

\newcommand{\ortho}{\mathbb{O}}
\newcommand{\igood}{\cI_{\mathsf{good}}}
\newcommand{\ibad}{\cI_{\mathsf{bad}}}

\newcommand{\ideal}{\mathsf{ideal}}
\newcommand{\corr}{\mathsf{corr}}
\newcommand{\frobnorm}[1]{\norm{#1}_{\mathrm{F}}}

\usepackage{hyperref}
\definecolor{mylinkcolor}{RGB}{0,0,140}
\hypersetup{colorlinks,allcolors=mylinkcolor,citecolor=mylinkcolor}

\allowdisplaybreaks

\title{Communication-efficient distributed eigenspace estimation with arbitrary node failures}

\author{%
  Vasileios Charisopoulos\\
  Department of Operations Research \& Information Engineering\\
  Cornell University\\
  Ithaca, NY 14853 \\
  \texttt{vc333@cornell.edu}
  \AND
  Anil Damle\\
  Department of Computer Science\\
  Cornell University\\
  Ithaca, NY 14853\\
  \texttt{damle@cornell.edu}
}

\begin{document}

\maketitle

\begin{abstract}
  We develop an eigenspace estimation algorithm for distributed environments
  with arbitrary node failures, where a subset of computing nodes can return
  structurally valid but otherwise arbitrarily chosen responses.
  Notably, this setting encompasses several important scenarios that arise in
  distributed computing and data-collection environments such as silent/soft errors, outliers or corrupted data at certain nodes, and adversarial responses.
  Our estimator builds upon and matches the performance of a recently proposed
  non-robust estimator 
  up to an additive $O(\sigma \sqrt{\alpha}$) error,
  where $\sigma^2$ is the variance of the existing estimator and $\alpha$ is the fraction
  of corrupted nodes.
\end{abstract}

\section{Problem overview and background}
Modern machine learning has seen the proliferation of heterogeneous
distributed environments for training and deploying data science
pipelines. As communication between machines is often the most time-consuming operation in
distributed systems, the design of \emph{communication-efficient algorithms} is
of paramount importance for scaling to massive datasets~\cite{ZCL+20}. However, the move to distributed environments also adds several additional layers of complexity in the design of algorithms. For example, in the distributed setting we would like our algorithms to be \emph{robust} and providing meaningful answers even in settings where some nodes contain outlier data~\cite{BNJ+10}, silently fail during the computation~\cite{mukherjee2005soft,reed2015exascale}, or are compromised and returning malicious results designed to corrupt the central solution. 

This work focuses on distributed \emph{eigenspace estimation} in the context
of robustness to node-level corruptions. Formally, we assume a computing environment
with nodes numbered $i = 1, \dots, m$, where every node $i$ observes a local
version $\est_i$ of an unknown symmetric matrix $\sol \in \Rbb^{d \times d}$;
the goal is  to approximate the subspace spanned by the $r \ll d$ principal
eigenvectors of $\sol$.
\textbf{Distributed PCA} is a standard example in this framework: every
machine draws $n$ i.i.d. samples from an unknown distribution $\cP$ with covariance
matrix $\sol$ and forms a local empirical covariance matrix $\est_i$. Recently
proposed communication-efficient algorithms have
every node $i$ transmit $V_i$, the $d \times r$ matrix of principal eigenvectors
of $\est_i$, to a central server, which then aggregates all the local solutions
via a carefully-crafted aggregation procedure~\cite{FWWZ19,CBD20}.

We devise and analyze an algorithm that is robust to a wide range of corruptions
that can occur to a subset of the computational nodes. In particular, we assume
that some fraction $\alpha$ of the computational nodes can respond with completely
arbitrary, but structurally valid, responses (i.e., they return arbitrary matrices
$V_i$ with orthonormal columns). This model encompasses three common forms of
node-level corruption that cannot be easily detected by the central machine in
isolation:
\begin{itemize}[leftmargin=*]
  \item[] \textbf{Silent/soft errors:} While computational errors may be rare
    on single machines, as distributed workloads span large numbers of nodes
    the probability that some of them fail becomes significant. Though
    catastrophic failures may be detectable, allowing the central server to
    simply ignore the output of specific nodes, the more nefarious issue is
    that of so-called silent (or soft) errors~\cite{dongarra2015fault,mukherjee2005soft,fiala2012detection}.
    More specifically, a silent error is one where a node returns an erroneous
    but structurally valid response to the central machine query. Because the
    response is structurally valid and the central machine may not have access
    to the per-node data it is not possible to ``validate'' the response of
    each node and, instead, the central estimator must be adapted to be robust
    to such errors.
  \item[] \textbf{Outliers or corrupted data:} In certain settings the data collection may be distributed in addition to the computation. If some of the nodes are drawing
    samples from an invalid or corrupted data source they may introduce gross
    outliers to the set of responses $\set{V_i \mid i\in [m]}$. Similarly, in the distributed
    PCA example, while most machines draw a sufficient number of samples, a
    minority of them may have only a small amount of data available such that the principal eigenspaces of the local
    empirical covariance matrices are too far from the ground truth, and thus
    violate standard modelling assumptions in distributed learning. Again,
    robustness to such outlier responses must be a feature of the
    estimator since they cannot be detected by individual nodes (as they do
    not have information about the global problem). 
  \item[] \textbf{Adversarial responses:} In some settings, a subset of nodes may be compromised by
    an adversary who wishes to influence the central solution by crafting and
    returning malicious $V_i$. In fact, the adversarial
    nodes may be collaborating when constructing their responses. 
    Since the central node does not get to see all the data it cannot validate responses or directly detect adversaries. Therefore, the estimator itself must be adapted to be robust to collections of responses designed to push the solution in specific directions. 
\end{itemize}
The main contribution of our paper is a \emph{communication-efficient algorithm that is robust to node corruptions} (as outlined above) for the distributed eigenspace estimation problem. We note in passing that our
corruption model is similar to so-called Byzantine failures~\cite{LSP82} in distributed systems.

\subsection{Related work}

\paragraph{Distributed eigenspace estimation.}
The problem of distributed eigenspace estimation has been well-studied in the
absence of malicious noise. One of the challenges in the distributed setting is
aggregating local solutions in the presence of symmetry: for example, if $v$
is an eigenvector of $\sol$, both $\pm v$ are valid solutions to our problem.
Various works deal with such symmetries in different ways; in the algorithms
of~\cite{BW19,FWWZ19}, the central node averages the \emph{spectral projectors}
of the local eigenspaces, and performs an eigendecomposition of the resulting
average to approximate the principal eigenspace. This approach is similar to the
algorithms of~\cite{LBK+14,CCH+16,BLS+16}, although the latter works focus on
distributed \emph{low-rank approximations} and do not address the issue of
approximating the principal eigenspace directly.
Another standard approach is for the central server to aggregate local
solutions after an alignment step designed to remove the orthogonal ambiguity
\cite{KA09,GSS17,CBD20} (see also~\cite{BAK08} for the non-distributed setting).
Indeed, our work builds on the two-stage algorithm presented and analyzed in
\cite{CBD20} for the non-robust setting.

Finally, we briefly mention a recent line of work~\cite{GSS17,CLL+20} that adapts
the \emph{shift-and-invert preconditioning} framework~\cite{GHJ+16} to the
distributed setting; however, the latter approach leads to algorithms that
require multiple rounds of communication.
\paragraph{Robust PCA.}
The literature contains a number of different formulations for robust
principal component analysis. The seminal work of Cand{\'e}s et al.~\cite{CLMW11}
formulated robust PCA as the task of separating an observed matrix $Y \in
\Rbb^{d \times d}$ into a low-rank and a sparse component -- a slightly
different problem from that considered in this paper. Xu et al.~\cite{XCM13}
considered the problem of approximating a low-dimensional distribution from a
set of $n$ i.i.d. samples, a constant fraction of which have been individually
corrupted by gross outliers. Follow-up works in the robust statistics
literature focused on sparse estimation in high dimensions and its application
to sparse robust PCA~\cite{BDL+17,DKK+19b}. However, to the best of our
knowledge, the existing literature on robust PCA does not
focus on communication-efficient estimators in the distributed setting. Indeed,
most related to ours is the line of work on byzantine-robust
distributed learning (typically focusing on distributed gradient descent);
see, e.g.,~\cite{CSX17,YCRB18,AAL18,PKH22,KLM22} as well as the survey~\cite{Fed+21}.
In these works, an iterative algorithm is distributed across machines
that send individual updates to a central server, which combines them using a
robust aggregation procedure (e.g., the geometric median~\cite{PKH22}). While
these works are more general in scope, they typically lead to estimators that
require multiple rounds of communication.
Instead, the algorithm we introduce in this paper will only require a single
communication step.

\subsection{Notation}
We let $\Sbb^{d-1}$ denote the unit sphere in $d$ dimensions. We write
$\frobnorm{A} := \sqrt{\ip{A, A}}$ and $\opnorm{A} := \sup_{x \in \Sbb^{d-1}}
\norm{Ax}_2$ for the Frobenius and
spectral norms of a matrix $A \in \Rbb^{n \times d}$. We write
$\ortho_{n, r}$ for the set of $n \times r$ matrices with orthonormal columns
and $\ortho_{r} \equiv \ortho_{r,r}$. Given $U, V \in \ortho_{n, r}$ we write
\begin{equation}
  \dist(U, V) := \opnorm{(I - UU^{\T}) V} =
  \opnorm{(I - VV^{\T}) U}
  \label{eq:dist-2}
\end{equation}
for their $\ell_2 \to \ell_2$ subspace distance and $\cB_{\dist}(U; r)$
for the scaled unit ball centered at $U$:
\[
  \cB(U; r) := \set{V \in \ortho_{d, r} \mid \dist(U, V) \leq r}.
\]
Finally, we use the notation $A \lesssim B$ to indicate that $A \leq c B$
for a dimension-independent constant $c > 0$ and $A \asymp B$ if $A \lesssim B$
and $B \lesssim A$ simultaneously.

\section{Robust distributed eigenspace estimation}
\label{sec:distributed-pca}
We now formally introduce the problem setting. In particular, we assume there
exists an \emph{unknown} symmetric matrix $\sol$ with spectral decomposition
\begin{equation}
    \sol = V \Lambda V^{\T} + V_{\perp} \Lambda_{\perp} V_{\perp}^{\T},
    \quad V \in \ortho_{n,r}, \; \Lambda = \diag(\set{\lambda_i(\sol)}_{i=1}^r), \;
    \Lambda_{\perp} = \diag(\set{\lambda_i(\sol)}_{i=r+1}^d),
    \label{eq:X-evdecomp}
\end{equation}
assuming a nonincreasing ordering on the eigenvalues:
\[
  \lambda_1(\sol) \geq \dots \geq \lambda_r(\sol) >
  \lambda_{r+1}(\sol) \geq \dots \geq \lambda_d(\sol).
\]
Our goal is to approximate the principal $r$-dimensional eigenspace
$\cV := \mathrm{span}(V)$ of $\sol$ given $m$ machines, each of which observes
a local version $\est_i$ of $\sol$, communicating with a central coordinator.
We assume that $m$ is even for simplicity.
When queried for a response,
machine $i$ responds either with an eigenvector matrix spanning the principal
eigenspace of the local matrix $\est_i$, or with an arbitrary $d \times r$
matrix with orthonormal columns. The latter case corresponds to so-called
\emph{compromised} machines.\ In contrast, prior work~\cite{FWWZ19,CBD20} assumes
that every machine responds truthfully.
\begin{assumption}[Corruption model]
  \label{asm:corruption}
  There exists a constant $\alpha \in (0, \nicefrac{1}{2})$ and an index set $\ibad
  \subset [m]$ with $\abs{\ibad} / m \leq \alpha$ such that the following holds:
  all nodes $i \notin \ibad$ observe a symmetric matrix $\est_i \in \Rbb^{d \times d}$.
  Moreover, when queried for a response, every node $i$ returns
  \begin{equation}
    \mtrv_i = \begin{cases}
      V_i, & i \in [m] \setD \ibad, \\
      Q_i, & i \in \ibad,
    \end{cases}
  \end{equation}
  where the columns of $V_i \in \ortho_{d,r}$ span the principal $r$-dimensional
  eigenspace of $\est_i$ and $Q_i \in \ortho_{d,r}$ is an arbitrary
  $d \times r$ matrix with orthonormal columns.
\end{assumption}
For notational convenience, we also define the set of ``good'' responses:
\begin{equation}
  \igood := [m] \setD \ibad, \quad
  \text{with} \quad
  \frac{\abs{\igood}}{m} \ge 1 - \alpha.
  \label{eq:igood}
\end{equation}

Furthermore, we require the principal eigenspace of $\sol$ to be sufficiently
separated from its complement and that the local errors $E_i := \est_i - \sol$
are not too large.
\begin{assumption}
    \label{asm:deterministic}
    There is a constant $\delta > 0$ such that the following hold:
    \begin{enumerate}
      \item \textbf{(Gap)} The matrix $A$ has a nontrivial eigengap:
      \begin{equation}
        \delta_{r}(A) := \lambda_{r}(\sol) - \lambda_{r+1}(\sol) \geq \delta.
        \label{eq:eigengap-defn}
      \end{equation}
      \item \textbf{(Approximation)} For all $i \in \igood$, the local observations satisfy:
      \begin{equation}
        \opnorm{A_i - A} \leq \frac{\delta_r(A)}{8}.
        \label{asm:deterministic:eq:approximation}
      \end{equation}
    \end{enumerate}
\end{assumption}
We note that the difficulty of the problem admits a natural proxy in the form of
the \emph{normalized inverse eigengap} $\kappa$, defined below:
\begin{equation}
  (\textbf{Normalized inverse eigengap}) \qquad
  \kappa := \frac{\opnorm{A}}{\delta_{r}(A)} =
  \frac{\lambda_1(A)}{\lambda_r(A) - \lambda_{r+1}(A)}.
  \label{eq:normalized-inv-gap}
\end{equation}
Our algorithm for the robust distributed eigenspace estimation problem is
outlined in~\Cref{alg:robust-distpca}, which is essentially a ``robust'' version
of the Procrustes fixing algorithm from~\cite{CBD20}.\ The latter (non-robust)
algorithm operates as follows: first, every machine $i$ computes its local
eigenvector matrix $V_i$ and broadcasts it to the central server. Because
invariant subspaces do not admit unique representations, naively averaging
these estimates can fail to reduce the approximation error further. Instead,
the algorithm of~\cite{CBD20} first picks one of the local solutions (say
$V_1$) as a reference and ``aligns'' every other solution with it by solving
a so-called \emph{Procrustes problem}:
\begin{equation}
  Z_i := \argmin_{U \in \ortho_r} \frobnorm{V_i U - V_1}, \quad
  i = 2, \dots, m.
  \label{eq:procrustes-problem}
\end{equation}
After the alignment step~\eqref{eq:procrustes-problem}, the solution of which is
available in closed form via the SVD~\cite{GVL13}, the central coordinator
computes and returns the empirical average $(1 / m) \sum_{i = 1}^m V_i Z_i$.

To robustify the algorithm described above against node failures, we need the
following ingredients:
\begin{description}[leftmargin=12pt]
  \item[Reference estimation.] In the presence of corruptions
    one must guard against the possibility of choosing an outlier as a
    reference solution (which would render the alignment step~\eqref{eq:procrustes-problem}
    useless). The first step of our algorithm robustly determines a reference
    guaranteed to have nontrivial alignment with the ground truth.
  \item[Solution aggregation.] With the robust reference at hand, the next step of the algorithm aligns other local solutions with it.
    However, since some of the solutions are outliers, we use a robust mean estimation
    algorithm in the last step of~\cref{alg:robust-distpca} to
    compute the empirical average only over inliers (and possible ``benign''
    outliers) with high probability.
\end{description}
We analyze each ingredient of~\cref{alg:robust-distpca} separately, in
\cref{sec:robust-reference-estimator,sec:procrustes-fixing,sec:robust-mean-estimation};
all proofs appear in the appendix.
Notably, our analysis is almost completely deterministic:
indeed, the only source of randomness is the filtering algorithm used in the
final stage (\cref{alg:adaptive-filter}).

\begin{algorithm}[h]
  \caption{Robust distributed eigenspace estimation}
  \begin{algorithmic}
    \State \textbf{Input}: responses $\set{\mtrv_i}_{i=1, \dots, m}$, corruption fraction $\alpha$,
    failure prob. $p$, error parameter $\omega$.
    \State $\mtrv_{\mathsf{ref}} := \mathtt{RobustReferenceEstimator}\left(
        \mtrv_{1}, \dots, \mtrv_{m}\right)$.
      \Comment{\cref{alg:robust-ref-estimator}; \cref{sec:robust-reference-estimator}}
    \State $\set{\widetilde{V}_{i}}_{i = 1, \dots, m} :=
    \mathtt{ProcrustesFixing}\left(
        \set{\mtrv_{1}, \dots, \mtrv_{m}}, \mtrv_{\mathsf{ref}}
    \right)$
    \Comment{\cref{alg:procrustes-fixing}; \cref{sec:procrustes-fixing}}
    \State $\bar{V} := \mathtt{AdaptiveFilter}(
      \set{\widetilde{V}_1, \dots, \widetilde{V}_m}, 6, \omega, p, \alpha)$
      \Comment{\cref{alg:adaptive-filter}; \cref{sec:robust-mean-estimation}}
    \State \Return $\bar{V}$
  \end{algorithmic}
  \label{alg:robust-distpca}
\end{algorithm}
Our main Theorem on the performance of~\cref{alg:robust-distpca} now follows.
\begin{theorem}
  \label{theorem:main-theorem}
  Let~\cref{asm:corruption,asm:deterministic} hold and suppose that the
  corruption level $\alpha$ satisfies
  \begin{equation}
    \varphi := \alpha + \frac{6 \log(1 / p)}{m} < \frac{1}{12}.
    \label{eq:breakdown-point}
  \end{equation}
  Then~\cref{alg:robust-distpca} returns an estimate $\bar{V} \in \Rbb^{d \times r}$ satisfying
  the following:
  \begin{equation}
    \dist(\bar{V}, V) \lesssim \begin{aligned}[t]
      &
      \underbrace{\frac{1}{\delta}
      \biggopnorm{\frac{1}{\abs{\igood}} \sum_{i \in \igood} \est_i - \sol}}_{E_{\mathsf{oracle}}}
      +
      \underbrace{\frac{\kappa^2}{\abs{\igood}} \sum_{i \in \igood}
      \left( \frac{\opnorm{\est_i - \sol}}{\delta} \right)^2}_{E_{\mathsf{high}}}
      + \underbrace{\sqrt{\varphi \max\left(
      \omega, \sigma^2 \right)}}_{E_{\mathsf{robust}}}.
    \end{aligned}
  \end{equation}
  with probability at least $1 - 2 \log\left(\nicefrac{6}{\omega}\right) \cdot p$.
  Moreover, the variance $\sigma^2$ satisfies
  \begin{equation}
    \sigma^2 \leq \Bigopnorm{\frac{1}{\abs{\igood}} \sum_{i \in \igood} V_i V_i^{\T} - VV^{\T}}
      + 2 \cdot \Bigopnorm{\frac{1}{\abs{\igood}} \sum_{i \in \igood} \widetilde{V}_i - V}.
  \end{equation}
\end{theorem}
The partition of the error in~\cref{theorem:main-theorem} admits a natural
interpretation: the first term, $E_{\mathsf{oracle}}$, corresponds to an
``oracle'' estimator that approximates $V$ via the principal eigenspace of
$\nicefrac{1}{\abs{\igood}} \sum_{i \in \igood} \est_i$. The second term,
$E_{\mathsf{high}}$, represents high-order errors that occur as a result
of the alignment step in~\cref{alg:procrustes-fixing}. Finally, the term
$E_{\mathsf{robust}}$ is the result of layering a robust mean estimation
algorithm on top of the alignment procedure and becomes negligible as the
fraction of corrupted nodes $\alpha \dto 0$.
We comment on the scaling of $E_{\mathsf{robust}}$ relative to the error
of the non-robust algorithm in the context of distributed PCA in~\cref{sec:pca-results}.

\subsection{The robust reference estimator}
\label{sec:robust-reference-estimator}
This section focuses on the analysis of Algorithm~\ref{alg:robust-ref-estimator},
which yields the robust reference estimator $\mtrv_{\mathsf{ref}}$ used to remove
the orthogonal ambiguity from local solutions. We note that the construction of
the estimator dates back to the seminal work of Nemirovski and Yudin~\cite{NY83}.
\begin{algorithm}
  \caption{$\mathtt{RobustReferenceEstimator}(Y_1, \dots, Y_m)$}
  \begin{algorithmic}
    \For{$i = 1, \dots, m$}
      \State $\varepsilon_i := \min\set{r \ge 0 \mid
      \abs{\cB_{\dist}(Y_i; r) \cap \set{Y_i}_{i=1}^m} > \frac{m}{2}}$
    \EndFor
    \State \Return $Y_{i_{\star}}$, where
    $i_{\star} := \argmin_{i \in [m]} \varepsilon_i$
  \end{algorithmic}
  \label{alg:robust-ref-estimator}
\end{algorithm}

\begin{remark}
  \label{remark:robust-ref-efficient}
  {\it
    The quantities $\varepsilon_i$ in~\Cref{alg:robust-ref-estimator} can be found
    in time $O(m^2 dr^2)$ by first computing $r_{j} := \dist(Y_i, Y_{j})$ for all $j
    \neq i$ and setting $\varepsilon_{i} := \mathtt{median}(\set{r_j}_{j \neq i})$.
  }
\end{remark}
Note that even though $\mtrv_{\mathsf{ref}}$ could be chosen among some of the
compromised samples, its construction ensures that it essentially inherits the
accuracy of the majority of the responses.
\begin{proposition}[Robust reference estimator]
  \label{proposition:robust-reference-estimator}
  Given a sample $\set{Y_1, \dots, Y_m}$ where
  $Y_i \in \ortho_{d, r}$ and
  $\abs{\set{i \in [m] \mid \dist(Y_i, V) \le \varepsilon}} > \nicefrac{m}{2}$
  for a fixed $\varepsilon > 0$, \cref{alg:robust-ref-estimator}
  outputs $Y_{i_{\star}}$ satisfying
  \begin{equation}
    \dist(Y_{i_{\star}}, V) \leq 3 \varepsilon.
    \label{eq:robust-ref-estimator}
  \end{equation}
\end{proposition}
\begin{proof}
  Define $\cC := \set{i \in [m] \mid \dist(Y_i, V) \leq \varepsilon}$
  with $\abs{\cC} > \frac{m}{2}$. Now, we consider
  any pair $(i, j)$ with $i, j \in \cC$. By the triangle inequality,
  \begin{equation}
    \dist(Y_i, Y_j) \leq \dist(Y_i, V) + \dist(Y_j, V) \leq 2 \varepsilon,
    \quad \text{for all $i, j \in \cC$.}
  \end{equation}
  Now, fix $i_{\star}$ to be any index for which $\dist(Y_i, Y_j) \leq 2 \varepsilon$
  for at least $m / 2$ other indices $j \neq i_{\star}$ (such an index always
  exists because $\abs{\cC} \ge \frac{m}{2} + 1$). For any such $i_{\star}$, there
  must be another index $j$ satisfying $\dist(Y_j, V) \le \varepsilon$ and
  $\dist(Y_j, Y_{i_{\star}}) \le 2 \varepsilon$. Therefore,
  \[
    \dist(Y_{i_{\star}}, V) \leq \dist(Y_{i_{\star}}, Y_j) +
    \dist(Y_j, V) \le 3 \varepsilon.
  \]
\end{proof}

\subsection{The \texttt{ProcrustesFixing} algorithm}
\label{sec:procrustes-fixing}
In this section, we formally introduce the Procrustes-fixing procedure and show
that it properly aligns all the non-compromised responses given the reference
solution described in~\cref{sec:robust-reference-estimator}. The procedure is
described in~\cref{alg:procrustes-fixing}; it accepts a set of $d \times r$
matrices with orthonormal columns as well as a reference matrix $Y_{\mathsf{ref}}$
of the same shape.

\begin{algorithm}
  \caption{$\mathtt{ProcrustesFixing}(\set{Y_1, \dots, Y_m},
  Y_{\mathsf{ref}})$}
  \begin{algorithmic}
    \For{$i = 1, \dots, m$}
        \State $\widetilde{Y}_i := Y_i Z_i, \quad \text{where} \quad
        Z_i := \argmin_{Z \in \ortho_r}\frobnorm{Y_i Z - Y_{\mathsf{ref}}}$
        \Comment{Procrustes alignment}
    \EndFor
    \State \Return $\set{\widetilde{Y}_i \mid i \in [m]}$
  \end{algorithmic}
  \label{alg:procrustes-fixing}
\end{algorithm}

The work~\cite{CBD20} provides an error bound for the \texttt{ProcrustesFixing}
algorithm under idealized conditions; namely, that the reference solution is
equal to the ground truth $V$.
\begin{theorem}[Theorem 2 in~\cite{CBD20}]
    \label{theorem:procrustes-fixing-works}
    Let~\cref{asm:deterministic} hold and let
    \[
      \tilde{V} := \frac{1}{\abs{S}} \sum_{i \in S} \widetilde{V}_i,
      \quad \text{where} \quad
      \set{\widetilde{V}_i}_{i \in S} = \mathtt{ProcrustesFixing}(\set{V_{i}}_{i \in S},
        V
      ), \;
      S \subset \igood.
    \]
    Then the following bound holds:
    \begin{equation}
      \opnorm{\tilde{V} - V}
        \lesssim
        \frac{1}{\delta^2} \frac{1}{\abs{S}} \sum_{i \in S} \opnorm{\est_i - \sol}^2
        + \frac{1}{\delta} \biggopnorm{\frac{1}{\abs{S}} \sum_{i \in S} \est_i - \sol}.
        \label{eq:procrustes-fixing-error}
    \end{equation}
\end{theorem}
While the setting of~\cref{theorem:procrustes-fixing-works} is idealized,
when the reference chosen by~\cref{alg:robust-ref-estimator} is sufficiently
close to $V$ one would expect that the aligned estimates are not far from their
ideal version.
The next Lemma shows that aligning the local solutions with $\mtrv_{\mathsf{ref}}$
is equivalent to aligning with the ground truth $V$, up to higher-order errors.
\begin{lemma} \label{lemma:reference-distance}
  Let $V_i \in \ortho_{d, r}$ span the principal $r$-dimensional eigenspace
  of the matrix $\est_i$ and let $V \in \ortho_{d, r}$ span the principal
  $r$-dimensional invariant subspace of $\sol$.
  Suppose that there is a $V_{\mathsf{ref}} \in \ortho_{d, r}$ satisfying
  \(
    \dist(V_{\mathsf{ref}}, V) = \varepsilon < \nicefrac{\delta_r(A)}{8},
  \)
  and define the sets of aligned estimates
  \[
    \begin{aligned}
      V_{i}^{\mathsf{ideal}} &:= V_i \cdot
        \argmin_{Z \in \ortho_r} \frobnorm{V_i Z - V}, \quad
      V_{i}^{\mathsf{corr}} := V_i \cdot \argmin_{Z \in \ortho_{r}} \frobnorm{V_i Z - V_{\mathsf{ref}}}.
    \end{aligned}
  \]
  Then for any $i \in \igood$ the following holds:
  \begin{equation}
    \opnorm{V_i^{\mathsf{ideal}} - V_i^{\mathsf{corr}}} \lesssim
    \frac{1}{\delta^2} \, \max\set{\opnorm{\est_i - A}^2, \opnorm{A}^2 \varepsilon^2}.
    \label{eq:aligned-diff-distance}
  \end{equation}
\end{lemma}
Putting everything together, we arrive at a \textit{deterministic}
characterization of the error attained by the empirical average over any subset
of responses that come from non-compromised nodes and have been aligned with the
robust reference estimator. Note that this characterization does not immediately
translate to an algorithm, since the set of compromised nodes is not known a-priori.
\begin{proposition}[Error of clean samples]
  \label{proposition:iid-sample-error}
  Let $\mtrv_{\mathsf{ref}}$ be the output of~\cref{alg:robust-ref-estimator}
  given inputs $\mtrv_{1}, \dots, \mtrv_{m}$. For any index set $S \subset
  \igood$ and $i \in S$, define
  \[
    V_i^{\corr} := V_i \cdot \argmin_{Z \in \ortho_r}
    \frobnorm{V_i Z - \mtrv_{\mathsf{ref}}}; \quad
    V_i^{\ideal} := V_i \cdot \argmin_{Z \in \ortho_r}
    \frobnorm{V_i Z - V}.
  \]
  Suppose that $\dist(V, \mtrv_{\mathsf{ref}}) = \varepsilon <
  \nicefrac{\delta_{r}(A)}{8}$. Then the following bound holds:
  \begin{equation}
    \biggopnorm{\frac{1}{\abs{S}} \sum_{i \in S} V_i^{\corr} - V}
    \lesssim
    \frac{1}{\delta^2 \abs{S}} \sum_{i \in S}
    \max\left(\opnorm{A_i - A}^2, \opnorm{A}^2 \varepsilon^2
    \right) + \frac{1}{\delta} \biggopnorm{\frac{1}{\abs{S}}\sum_{i \in S} \est_i - \sol}.
    \label{eq:procrustes-fixing-ideal-samples}
  \end{equation}
\end{proposition}
\begin{proof}
  From the triangle inequality, \cref{lemma:reference-distance} and~\cref{theorem:procrustes-fixing-works}
  it follows that
  \begin{align*}
    \biggopnorm{\frac{1}{\abs{S}} \sum_{i \in S} {V}_i^{\corr} - V} &=
    \biggopnorm{\frac{1}{\abs{S}} \sum_{i \in S} {V}_i^{\corr} - {V}_i^{\ideal}
    + {V}_i^{\ideal} - V} \\
                                                              &\leq
    \biggopnorm{\frac{1}{\abs{S}} \sum_{i \in S} {V}_i^{\corr} - {V}_i^{\ideal}} +
    \biggopnorm{\frac{1}{\abs{S}} \sum_{i \in S} {V}_i^{\ideal} - V} \\
                                                              &\lesssim
    \frac{1}{\delta^2 \abs{S}} \sum_{i \in S} \max\left(
        \opnorm{\est_i - \sol}^2,
        \opnorm{\sol}^2 \varepsilon^2
    \right)
    + \biggopnorm{\frac{1}{\abs{S}} \sum_{i \in S} V_i^{\ideal} - V} \\
                                                              &\lesssim
    \frac{1}{\delta^2 \abs{S}} \sum_{i \in S} \max\left(
        \opnorm{\est_i - \sol}^2,
        \opnorm{\sol}^2 \varepsilon^2
    \right)
    + \frac{1}{\delta} \biggopnorm{\frac{1}{\abs{S}} \sum_{i \in S} \est_i - \sol}.
  \end{align*}
\end{proof}

\subsection{Analysis of robust mean estimation}
\label{sec:robust-mean-estimation}
We now analyze the last phase of the algorithm, which computes
an estimate of $V$ via the robust mean of the aligned samples.
The mean estimation procedure used is the randomized iterative filtering
method shown in~\cref{alg:filter}, the guarantees of which are summarized
in~\cref{theorem:filter-algorithm-guarantees}. Since it is natural to measure
error using the spectral norm, we extend the analysis of~\cite{PBR19} which
is applicable when error is measured in the Euclidean norm; complete proofs
are provided in the appendix.
\begin{algorithm}[h]
  \caption{\texttt{Filter}($S := \set{X_i}_{i=1, \dots, m}$, $\lub$)}
  \begin{algorithmic}
    \State Compute empirical mean and covariance:
    \[
      {\theta}_{S} := \frac{1}{\abs{S}} \sum_{i \in S} X_i, \quad
      \Sigma_{S} := \frac{1}{\abs{S}} \sum_{i \in S} (X_i - {\theta}_{S})
      (X_i - {\theta}_{S})^{\T}.
    \]
    \State Compute leading eigenpair $(\lambda, v)$ of $\Sigma_{S}$.
    \If{$\lambda < 18 \lub$}
      \State \Return ${\theta}_{S}$
    \Else
      \State Compute outlier scores
        $\tau_i := v^{\T} (X_i - {\theta}_{S})(X_i - {\theta}_{S})^{\T} v$ for $i \in S$.
      \State Sample $Z$ from $S$ following
        \(
          \prob{Z = X_i} = \frac{\tau_i}{\sum_{j \in S} \tau_j}.
        \)
      \State \Return \texttt{Filter}($S \setD \set{Z}$, $\lub$)
    \EndIf
  \end{algorithmic}
  \label{alg:filter}
\end{algorithm}
\begin{theorem}
  \label{theorem:filter-algorithm-guarantees}
  Suppose $G_0 \subset [m]$, $\alpha$ and $p \in (0, 1)$ satisfy
  \begin{equation}
    \alpha + \frac{6 \log(1 / p)}{m} \leq \frac{1}{12} \quad
    \text{and} \quad
    \abs{G_0} \geq (1 - \alpha) m.
    \label{eq:filter-algorithm-condition}
  \end{equation}
  Then if $\lub \geq \opnorm{\Sigma_{G_0}}$, \cref{alg:filter} returns
  an estimate ${\theta}_{\lub}$ satisfying
  \begin{equation}
    \prob{\Bigl\|{\theta}_{\lub} - \frac{1}{\abs{G_0}} \sum_{i \in G_0} X_i\Bigr\|_2
    \geq 18 \sqrt{5 \lub} \left(
      \alpha + \frac{4\log(1 / p)}{m}
    \right)^{1/2}} \leq p.
    \label{eq:approx-error-general-gamma}
  \end{equation}
\end{theorem}
The error in~\eqref{eq:approx-error-general-gamma} scales with the upper bound $\lub$,
which may be far from the ``optimal'' $\|\Sigma_{G_0}\|_2$. We
describe an adaptive version of~\cref{alg:filter} that achieves this at a logarithmic
additional cost. 
Indeed, suppose an upper bound on $\alpha$ is available and the unknown
parameter $\opnorm{\Sigma_{G_0}}$ lies in an interval $[\llb, \lub]$.
We construct a search grid $\cG$ as follows:
\begin{equation}
  \cG := \set{2^j \mid j \in \set{j_{\mathsf{lo}}, j_{\mathsf{hi}}}},
  \quad j_{\mathsf{lo}} := \floor{\log_2(\llb)}, \; \;
        j_{\mathsf{hi}} := \ceil{\log_2(\lub)}.
  \label{eq:search-grid}
\end{equation}
We are now in good shape to describe our estimator. To simplify notation,
we define the error proxy
\begin{equation}
  f(\lambda; p, \alpha) := 18 \sqrt{5\lambda} \left(
    \alpha + \frac{4 \log(1 / p)}{m}
  \right)^{1/2}.
  \label{eq:f-gamma}
\end{equation}
Our estimator, $\theta_{\hat{\lambda}}$, is implemented in Alg.~\ref{alg:adaptive-filter}
and defined as:
\begin{equation}
  \theta_{\hat{\lambda}}, \; \text{where} \;
  \hat{\lambda} := \min\set{\lambda \in \cG \mid
    \opnorm{{\theta}_{\lambda} - {\theta}_{\lambda'}}
    \leq f(\lambda; p, \alpha) + f(\lambda'; p, \alpha), \;
    \forall \lambda' \in \cG \cap [\lambda, \infty)}.
  \label{eq:meta-estimator}
\end{equation}
\begin{algorithm}[H]
  \caption{\texttt{AdaptiveFilter}(%
    $S = \set{X_i}_{i=1, \dots, m}$, $\lambda_{\mathsf{ub}}$, $\lambda_{\mathsf{lb}}$, $p$, $\alpha$)}
  \setstretch{1.25}
  \begin{algorithmic}
    \State Set up search grid:
    \(
      j_{\mathsf{lo}} := \floor{\log_2 \llb}, \;
      j_{\mathsf{hi}} := \ceil{\log_2 \lub}.
    \)
    \For{$j = j_{\mathsf{hi}}, \dots, j_{\mathsf{lo}}$}
      \State ${\theta}_{2^j} = \mathtt{Filter}(S, 2^j)$
      \Comment{\cref{alg:filter}}
      \If{$\exists k > j$ such that $\norm{{\theta}_{2^j} -
        {\theta}_{2^{k}}} > f(2^j; p, \alpha) + f(2^{k}; p, \alpha)$}
        \Comment{$f$ defined in~\eqref{eq:f-gamma}}
        \State \Return ${\theta}_{2^{j + 1}}$
      \EndIf
    \EndFor
    \State \Return ${\theta}_{2^{j_{\mathsf{lo}}}}$
  \end{algorithmic}
  \label{alg:adaptive-filter}
\end{algorithm}
If $\opnorm{\Sigma_{G_0}} \in [\llb, \lub]$, the estimator attains the
optimal error up to a constant while the success probability degrades only
logarithmically, as shown by the following Proposition.
\begin{proposition}
  \label{prop:meta-estimator-close}
  If $G_0$, $p$, and $\alpha$ satify~\eqref{eq:filter-algorithm-condition}
  and $\opnorm{\Sigma_{G_0}} \leq \lub$, the estimator
  $\theta_{\hat{\lambda}}$ from~\eqref{eq:meta-estimator} satisfies
  \[
    \Bigopnorm{{\theta}_{\hat{\lambda}} -
    \frac{1}{\abs{G_0}} \sum_{i\in G_0} X_i}
    \leq 171 \sqrt{\max\set{\opnorm{\Sigma_{G_0}}, \llb}} \left(
      \alpha + \frac{4 \log(1 / p)}{m}
    \right)^{1/2}
  \]
  with probability at least
  \(
    1 - 2\log_2\left(\nicefrac{\lub}{\llb}\right) p.
  \)
\end{proposition}
We suspect that the restriction $\alpha < \nicefrac{1}{12}$ is an artifact
of the proof; numerical evidence in~\cref{sec:experiments} suggests that
the breakdown point of Alg.~\ref{alg:adaptive-filter} is closer to the
natural limit of $\nicefrac{1}{2}$.

\subsection{Proof sketch of main theorem}
We briefly sketch the proof of~\cref{theorem:main-theorem} here. We decompose
\begin{align*}
  \dist(\bar{V}, V) \lesssim \opnorm{\bar{V} - V} &\leq
  \underbrace{
    \Bigopnorm{\bar{V} - \frac{1}{\abs{\igood}} \sum_{i \in \igood} \widetilde{V}_i}
  }_{\Delta_{1}} +
  \underbrace{
    \Bigopnorm{\frac{1}{\abs{\igood}} \sum_{i \in \igood} \widetilde{V}_i - V}
  }_{\Delta_2}
\end{align*}
The error $\Delta_1$ can be directly controlled by applying~\cref{prop:meta-estimator-close}
with $G_0 \equiv \igood$ combined with the fact that the spectral norm of the empirical
covariance $\Sigma_{\igood}$ admits the upper bound in~\eqref{eq:emp-cov-bound}; we
refer the reader to~\cref{lemma:good-set-covariance} in the appendix for a
complete statement and proof.
\begin{equation}
  \opnorm{\Sigma_{\igood}} \leq \biggopnorm{
    \frac{1}{\abs{\igood}} \sum_{i \in \igood}
    V_i V_i^{\T} - VV^{\T}
    } + 2 \, \biggopnorm{\frac{1}{\abs{\igood}} \sum_{i \in \igood}
    \widetilde{V}_i - V}
    \label{eq:emp-cov-bound}
\end{equation}
Finally, we control the error $\Delta_2$ by invoking
\cref{proposition:iid-sample-error} with $S = \igood$, since
\[
  \widetilde{V}_i = V_i \cdot \argmin_{Z \in \ortho_r} \frobnorm{V_i Z - \mtrv_{\mathsf{ref}}},
  \quad \text{for all } i \in \igood.
\]
Combining the resulting upper bounds yields the error in~\cref{theorem:main-theorem}.
\hfill
\qedsymbol

\begin{remark}
{\it Both of the terms in~\eqref{eq:emp-cov-bound} are typically small
and can be directly controlled in concrete applications such as
distributed PCA. Note that even though the bound in~\eqref{eq:emp-cov-bound}
is not directly computable, it is immediate that $\| \Sigma_{\igood} \|_2 \lesssim 1$
and thus we may initialize~\cref{alg:adaptive-filter} with $\lub = O(1)$ in the
absence of a finer upper bound.}
\end{remark}


\section{Robust distributed PCA}
\label{sec:pca-results}
In this section, we specialize the results of~\cref{sec:distributed-pca} to
robust distributed PCA for subgaussian distributions. We first formalize the
sampling model for the problem.
\begin{assumption}[Subgaussian data]
  Every machine $i \in \igood$ draws
  $\{X_j^{(i)}\}_{j=1}^n \iid \cP$, where $\cP$ is a zero-mean,
  subgaussian distribution with covariance matrix $\sol := \expec[X \sim \cP]{XX^{\T}}$,
  and forms $A_i := \frac{1}{n} \sum_{j=1}^n X_j^{(i)} (X_j^{(i)})^{\T}$.
  \label{assumption:distributed-pca}
\end{assumption}
Our main theorem follows directly from~\cref{theorem:main-theorem} and control of
$\bigopnorm{\Sigma_{\igood}}$ under~\cref{assumption:distributed-pca}.
\begin{theorem}
  \label[theorem]{theorem:distributed-pca-rate}
  Let~\cref{asm:deterministic,asm:corruption,assumption:distributed-pca} hold and suppose that
  \(
    n \gtrsim \kappa^2 \cdot \left(r_{\star} + \log(mn/p)\right)
  \) and $\alpha$, $m$ and $p$ satisfy~\eqref{eq:breakdown-point}.
  Then~\cref{alg:robust-distpca} initialized with
  $\omega = \sqrt{\nicefrac{1}{mn}}$ returns a $\bar{V}$ satisfying
  \begin{equation}
    \dist(\bar{V}, V) \lesssim
      \sqrt{\varrho \left(
          \alpha + \frac{\log(1 / p)}{m}
      \right)}
      + \kappa \sqrt{
        \frac{r(r_{\star} + \log(n))}{(1 - \alpha) mn}
      } + \kappa^4 \cdot
      \frac{r(r_{\star} + \log(mn / p))}{n},
    \label{eq:error-bound-pca}
  \end{equation}
  with probability at least $1 - \nicefrac{2}{n} - \log_2(12mn) p$.
  Here, $r_{\star} := \trace{A} / \opnorm{A}$ and $\varrho$ is given by
  \[
    \varrho := \kappa \sqrt{\frac{r(r_{\star} + \log(n))}{(1 - \alpha) mn}}
    + \max\set{\kappa^2 \sqrt{r}, \kappa^4} \cdot
    \frac{r_{\star} + \log(nm/p)}{n}.
  \]
\end{theorem}
When $\kappa \asymp 1$, high-order terms in~\cref{theorem:distributed-pca-rate}
can be discarded and we arrive at the following:
\begin{corollary}
  Assume that the conditions of~\cref{theorem:distributed-pca-rate} hold and
  $\kappa \asymp 1$. Then:
  \[
    \dist(\bar{V}, V) \lesssim \sqrt{\alpha + \frac{\log(1 / p)}{m}} \cdot
      \left(\frac{r(r_{\star} + \log(n))}{(1 - \alpha) mn}\right)^{1/4} +
      \sqrt{\frac{r(r_{\star} + \log(n))}{(1 - \alpha) mn}}
  \]
  with probability at least $1 - \nicefrac{2}{n} - \log_2(12mn)p$.
\end{corollary}
We briefly compare the error of~\cref{alg:robust-distpca} to
that of its non-robust counterpart from~\cite{CBD20} when $\kappa
\asymp 1$. The latter algorithm returns an estimate $\bar{V}^{\mathsf{nonrobust}}$ satisfying
\(
  \dist(\bar{V}^{\mathsf{nonrobust}}, V) =
  \tilde{\cO}\left(\sqrt{\nicefrac{r_{\star}}{mn}}\right).
\)
Ignoring the $\sqrt{r}$ factors, which are likely an artifact of our proof,
our algorithm also introduces an additive error of the order
\(
  \tilde{\cO}(
    \sqrt{\nicefrac{\alpha}{1 - \alpha}} \cdot
    \left( \nicefrac{r_{\star} r}{mn} \right)^{1/4}
  ).
\)
Note that for a constant absolute number of corruptions
$\alpha \propto \nicefrac{1}{m}$, this additive factor scales as
\begin{equation*}
  \sqrt{\frac{\alpha}{1 - \alpha}} \left(\frac{r_{\star} r}{mn}\right)^{1/4}
  \lesssim
  \left(
    \frac{r_{\star} r}{m^3 n}
  \right)^{1/4}.
\end{equation*}
If $m$ and $n$ are comparable, this is similar to the error of the non-robust
algorithm up to an $(r / r_{\star})^{1/4}$ factor. Therefore, the performance
of~\cref{alg:robust-distpca} degrades gracefully as a function of the corruption
level under not too restrictive assumptions on the ratio $m / n$.

\subsection{Numerical study}
\label{sec:experiments}
We provide a brief numerical illustration of the performance of
\cref{alg:robust-distpca} on data sampled from an unknown Gaussian distribution
$\cD := \cN(0, V \Lambda V^{\T} + V_{\perp} \Lambda_{\perp} V_{\perp}^{\T})$,
where $\bmx{V & V_{\perp}} \in \ortho_d$ is a random $d \times d$ orthogonal matrix
and $\Lambda, \Lambda_{\perp}$ are generated according to the following model:
\begin{equation}
  \Lambda = I_{r}, \; \;
  (\Lambda_{\perp})_{jj} = (1 - \delta) \eta^j, \; \; j = 1, \dots, d - r,
  \quad \text{where} \quad
  \eta = 1 - \frac{1 - \delta}{r_{\star} - r} \in (0, 1).
  \label{eq:synthetic-model}
\end{equation}
We simulate an adversary by replacing the first $\floor{
\alpha m}$ responses by the same $V_{\mathsf{adv}} \in
\ortho_{d,r}$, chosen to be near-orthogonal to $V$. We fix
the gap $\delta = 0.25$ throughout. We compare Alg.~\ref{alg:robust-distpca} (labelled \texttt{Robust}
in our plots) against two baselines:
the algorithm from~\cite{CBD20} (labelled \texttt{Naive}),
which corresponds to Alg.~\ref{alg:procrustes-fixing} using the first response -- which is always corrupted in our experiment -- as the reference followed by naive averaging; and a version of
Alg.~\ref{alg:robust-distpca} without the robust mean estimation step
(labelled \texttt{Procrustes}).
Our implementation always removes the sample with the largest outlier score
in each step of Alg.~\ref{alg:filter} and uses a simplified error proxy $f(\lambda; \alpha)
:= \sqrt{\lambda \alpha}$ instead of~\eqref{eq:f-gamma} in Alg.~\ref{alg:adaptive-filter}.

Our experiment is illustrated in~\Cref{fig:synthetic_r1-5}. Clearly, the
baseline methods break down in the presence of corruption, yielding
solutions nearly orthogonal to $V$ as $\alpha$
approaches $\nicefrac{1}{2}$. In contrast, the error of Alg.~\ref{alg:robust-distpca}
degrades gracefully with $\alpha$. We note that our algorithm
yields a nontrivial solution even when almost half of the measurements are
corrupted ($\alpha = 45\%$), in line with intuition suggesting that
$\alpha_{\star} = \nicefrac{1}{2}$ is a natural breakdown point for
outlier-robust algorithms.

\begin{figure}[h]
  \centering
  \includegraphics[width=0.9\linewidth]{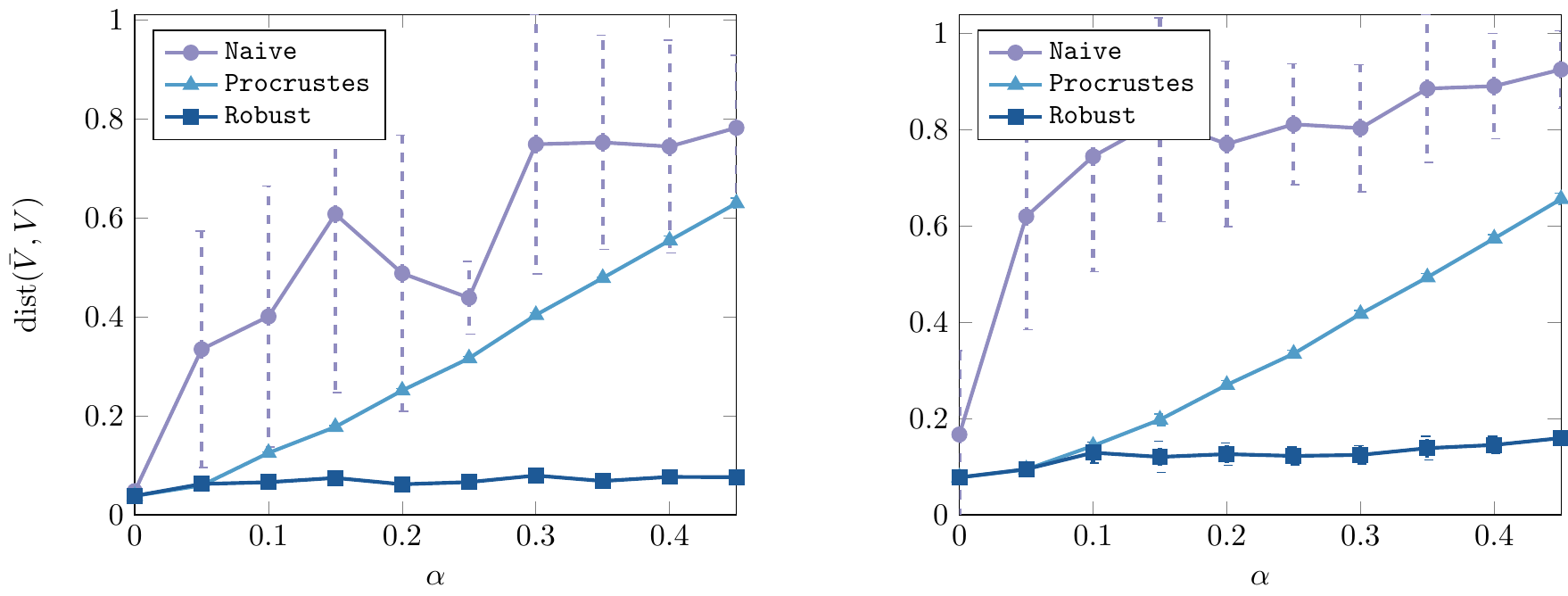}
  \caption{Robust distributed PCA with $m = 150$, $n = 50r$, $r_{\star} = 2r$,
    and $\kappa = 5$ under Model~\eqref{eq:synthetic-model}.
    We report the mean subspace distance $\pm$ one standard deviation over $10$ independent runs
  for subspace dimension $r = 5$ (\textbf{left}) and $r = 10$ (\textbf{right}).}
  \label{fig:synthetic_r1-5}
\end{figure}

\section{Discussion}

We presented a communication-efficient algorithm for distributed eigenspace
estimation that is robust to compromised nodes returning structurally valid but
otherwise potentially adversarial responses.
While theory predicts that our algorithm is able to handle a constant
corruption level $\alpha < \nicefrac{1}{12}$, numerical evidence suggests its
breakdown point is closer to the (optimal) $\alpha_{*} = \nicefrac{1}{2}$,
which might be achievable by an improved analysis of the filtering procedure
in Alg~\ref{alg:filter}.

Our adaptive version of the filtering procedure in
Alg.~\ref{alg:adaptive-filter} trades off knowledge of (an upper bound on)
the corruption level $\alpha$ with the need for a precise bound on
$\| \Sigma_{\igood} \|_2$. In the complementary situation (where such a
bound on $\| \Sigma_{\igood} \|_2$ is known), one can design a version
of~\cref{alg:adaptive-filter} that is adaptive to the corruption level
$\alpha$ using a similar construction that evaluates the error proxy
$f(\lambda; p, \alpha)$ for different values of $\alpha$ and fixed
$\lambda \approx \| \Sigma_{\igood} \|_2$ instead.

Finally, we note that our algorithm suggests a natural pipeline for robustifying
communication-efficient one-shot algorithms by aggregating local responses
after an outlier filtering stage, which is likely applicable to other
statistical problems admitting one-shot estimators in the distributed setting.

\subsection*{Acknowledgements}
We thank Jayadev Acharya and Damek Davis for their insightful comments and suggestions.

\bibliographystyle{hplain}
\bibliography{main}

\clearpage

\appendix

\section{Auxiliary results}
In this section, we present a few supporting results. The first
result is a path independence lemma for perturbations of eigenvectors.
It first appeared in~\cite{Stewart12}; the eigengap condition in the
statement of the Lemma is justified in~\cite[Lemma 5]{CBD20}.
\begin{lemma}[Path independence]
  \label[lemma]{lemma:path-independence}
  Let $A \in \Rbb^{d \times d}$ be a fixed \emph{symmetric} matrix and let
  $\hat{A} := A + E$, where $E$ is a symmetric perturbation. Suppose that we
  can write
  \[
    \hat{A} - A = E_0 + E_1 = F_0 + F_1,
  \]
  where $E_0, E_1, F_0, F_1$ are symmetric matrices, and define
  the \emph{intermediate} matrices
  \[
    \hat{A}_1 := A + E_0, \; \hat{A}_2 = \hat{A}_1 + E_1, \quad
    \tilde{A}_1 := A + F_0, \; \tilde{A}_2 = \tilde{A}_1 + F_1.
  \]
  Fix any $V \in O(d, r)$ whose columns span the principal $r$-dimensional
  invariant subspace of $A$ and construct the leading eigenvector matrices
  $\hat{V}_1$, $\hat{V}_2 \in O(d, r)$ of $\hat{A}_1$ and $\hat{A}_2$ such that
  \[
    \min_{U \in \Obb_{r}} \norm{\hat{V}_1 U - V}_F = \norm{\hat{V}_1 - V}_F,
    \quad
    \min_{U \in \Obb_{r}} \norm{\hat{V}_2 U - \hat{V}_1}_F
    = \norm{\hat{V}_2 - \hat{V}_1}_F.
  \]
  Further, let $\tilde{V}_1$ and $\tilde{V}_2$ be the leading eigenvector
  matrices of $\tilde{A}_1$ and $\tilde{A}_2$, constructed in a similar fashion.
  Then, $\hat{V}_2$ and $\tilde{V}_2$ both span principal invariant subspaces
  of $\hat{A} = A + E$. Moreover, they satisfy
  \[
    \hat{V}_2 = \tilde{V}_2 + T, \quad
    \opnorm{T} \lesssim \frac{\varepsilon^2}{\delta^2}, \quad
    \varepsilon := \max\set{\opnorm{E_0}, \opnorm{E_1},
    \opnorm{F_0}, \opnorm{F_1}},
  \]
  as long as $A$ satisfies $\delta_{r}(A) \geq 4 \varepsilon$.
\end{lemma}

\begin{lemma}
  \label[lemma]{lemma:reverse-davis-kahan}
  Suppose that $U \in \ortho_{d, r}$ satisfies $\dist(U, V) \leq \varepsilon
  < \nicefrac{1}{2}$, where $V$  is the principal eigenvector matrix of a symmetric
  matrix $A$ with eigengap $\delta_{r}(A) := \lambda_r(A) - \lambda_{r+1}(A) >
  0$. Then there exists a symmetric matrix $B$ such that the following hold:
  \begin{enumerate}
    \item $\opnorm{A - B} \leq 8 \opnorm{A} \varepsilon$ and
    $\delta_{r}(B) = \delta_{r}(A)$.
    \label{lemma:item:approximation}
    \item $U$ is the principal eigenvector matrix of $B$.
    \label{lemma:item:principal}
  \end{enumerate}
\end{lemma}
\begin{proof}
  We prove Item~\ref{lemma:item:approximation} first. To that end, we can write
  $A = A_1 + A_2$, where
  $A_1 := V\Sigma_1 V^{\T}$ and $A_2 := V_{\perp} \Sigma_2 V_{\perp}^{\T}$.
  We consider the following matrix $B$:
  \begin{equation}
      B = U \Sigma_1 U^{\T} + U_{\perp} \Sigma_2 U_{\perp}^{\T},
      \label{eq:perturbed-matrix}
  \end{equation}
  where $U_{\perp}^{\T} U = 0$ and $U_{\perp} \in \ortho_{d, d-r}$.
  From~\eqref{eq:perturbed-matrix} and the condition $\Sigma_1 \succ \Sigma_2$,
  it follows that $U$ is a principal eigenvector matrix for $B$. Moreover, the
  gap condition on $A$ immediately translates to the claimed gap condition for $B$.
  
  It remains to bound the distance between $A$ and $B$. We write
  \begin{align*}
    \opnorm{A - B} &\leq \opnorm{U \Sigma_1 U^{\T} - A_1} +
    \opnorm{U_{\perp} \Sigma_2 U_{\perp}^{\T} - A_2}.
  \end{align*}
  To upper bound the first term on the right-hand side above, we use the spectral
  projectors $P_{U} := UU^{\T}$ and $P_{U_{\perp}} = I - P_U$ to decompose it
  into
  \begin{align*}
      \opnorm{U \Sigma_1 U^{\T} - A_1} &\leq
      \opnorm{U \Sigma_1 U^{\T} - P_{U} A_1} +
      \opnorm{P_{U_{\perp}} A_1} \\
      &\leq
      \opnorm{U \Sigma_1 U^{\T} - P_{U} A_1 P_U} +
      \opnorm{P_U A_1 P_{U_{\perp}}} + \opnorm{P_{U_{\perp}} A_1} \\
      &\leq \opnorm{\Sigma_1 - U^{\T} V \Sigma_1 V^{\T} U}
      + 2 \opnorm{\Sigma_1} \opnorm{P_{U_{\perp}} V} \\
      &\leq \opnorm{\Sigma_1(I - V^{\T} U)}
      + \opnorm{(I - V^{\T} U) \Sigma_1 V^{\T} U} +
      2 \opnorm{\Sigma_1} \opnorm{P_{U_{\perp}} V} \\
      &\leq 2 \opnorm{\Sigma_1} \varepsilon + 4 \opnorm{\Sigma_1} \varepsilon^2,
  \end{align*}
  where the last inequality follows from the inequality
  $\opnorm{P_{U_{\perp}} V} = \dist(U, V) = \varepsilon$ and
  \cref{lemma:modified-sin-theta}.
  A similar argument shows that
  \(
    \opnorm{U_{\perp} \Sigma_2 U_{\perp}^{\T} - A_2} \leq
    2 \opnorm{\Sigma_2} \varepsilon + 4 \opnorm{\Sigma_2} \varepsilon^2.
  \)
  Taking into account the bound $\varepsilon < 1/2$ completes the proof.
\end{proof}

\begin{lemma}[Modified $\sin \theta$ distance]
  \label{lemma:modified-sin-theta}
  Let $U, V \in O(d, r)$ satisfy $\dist(U, V) = \alpha < 1$. Then
  the following holds:
  \[
    \opnorm{I - U^{\T} V} \leq 2 \alpha^2.
  \]
\end{lemma}
\begin{proof}
  Let $P \Sigma Q^{\T}$ be the singular value decomposition of $U^{\T} V$.
  Recall that~\cite[Eq. (2.5)]{CCFM21}:
  \[
    \Sigma = \mathrm{diag}\left(\sigma_1, \dots, \sigma_r\right); \quad
      \sigma_i = \cos(\theta_i),
  \]
  where $\theta_i \in [0, \pi/2]$ and $\opnorm{\sin \Theta} = \alpha$,
  following~\cite[Lemma 2.5]{CCFM21}. From our assumptions, it follows
  that
  \begin{align*}
    \opnorm{I - U^{\T} V} &= \opnorm{P(I - \Sigma)Q^{\T}} \\
                          &= \opnorm{I - \Sigma} \\
                          &= \max_{i \in [r]} \set{ 1 - \cos(\theta_i) } \\
                          &= \max_{i \in [r]} \set{ 2 \sin^2 (\theta_i / 2) } \\
                          &\leq 2 \, \max_{i \in [r]} \sin^2 (\theta_i) \\
                          &= 2 \opnorm{\sin \Theta}^2,
  \end{align*}
  with the last inequality following from $0 \leq \sin(\theta / 2) \leq
  \sin(\theta)$ for any $\theta \in [0, \pi/2]$.
\end{proof}

\section{Omitted proofs}
\label{sec:appendix-mean-estimation}
This section includes proofs that were omitted from the main text.

\subsection{Proof of~\cref{lemma:reference-distance}}
\begin{proof}
  Recall that $V$ is an eigenvector matrix of $\sol$ that satisfies
  \[
    \min_{Z \in \ortho_r} \frobnorm{V - V_{\mathsf{ref}} Z} =
    \frobnorm{V - V_{\mathsf{ref}}}.
  \]
  From~\cref{lemma:reverse-davis-kahan}, it follows that the columns of $V_{\mathsf{ref}}$
  span the principal eigenspace of a matrix $B$ with nontrivial eigengap that
  satisfies $\opnorm{A - B} \lesssim \opnorm{A} \varepsilon$. We now relate
  $V^{\mathsf{ideal}}_i$ to $V_{i}^{\mathsf{corr}}$ using the aforementioned
  path independence result.

  To that end, note that ${V}_i^{\mathsf{ideal}}$ is the leading eigenvector
  matrix of
  \[
    A_i := A + (A_i - A) + 0,
  \]
  that has been maximally aligned with $V$ (in the sense of Frobenius distance).
  On the other hand, the Procrustes estimates ${V}_i^{\mathsf{corr}}$ are given
  by the leading eigenvector matrices of
  \[
    A_i := A + (B - A) + (A_i - B),
  \]
  since $V_{\mathsf{ref}}$ is the leading eigenvector of $B$ nearest to $V$
  and $V_i^{\mathsf{corr}}$ is formed as the leading eigenvector of $A_i$
  nearest to $V_{\mathsf{ref}}$. Applying~\cref{lemma:path-independence} with
  $E_0 := A_i - A$, $E_1 = \bm{0}$, $F_0 := B - A$ and $F_1 := A_i - B$, we obtain
  \[
    V_{i}^{\mathsf{corr}} = V_i^{\mathsf{ideal}} +
    \cO\left(\frac{1}{\delta^2}\max\set{
        \opnorm{A_i - A}^2, \opnorm{B - A}^2, \opnorm{A_i - B}^2
    }\right).
  \]
  Finally, we note the following upper bound
  \[
    \opnorm{A_i - B}^2 \lesssim \opnorm{A_i - A}^2 + \opnorm{A - B}^2
    \lesssim \max\left(\opnorm{A_i - A}^2, \opnorm{A}^2 \varepsilon^2\right),
  \]
  which concludes the proof.
\end{proof}

\subsection{Proof of~\cref{prop:meta-estimator-close}}
\begin{proof}
  Let ${j_{\mathsf{crit}}}$ be the smallest index for which
  $2^{j} \geq \max\set{\llb, \opnorm{\Sigma_{G_0}}}$. For a fixed corruption fraction $\alpha$
  and failure probability $p$, define the events
  \[
    \cE_j := \set{\Bigopnorm{\theta_{2^j} - \frac{1}{\abs{G_0}}
    \sum_{i \in G_0} X_i} \leq f(2^j; p, \alpha)}, \quad
    \cE := \setI_{j = j_{\mathsf{crit}}}^{j_{\mathsf{hi}}} \cE_j.
  \]
  From Theorem 3 in the main text and a union bound, it follows that
  \begin{align*}
    \prob{\cE} &\geq
    1 - \sum_{j \in \set{j_{\mathsf{crit}}, \dots, j_{\mathsf{hi}}}}
    \prob{\Bigopnorm{{\theta}_{2^j} -
    \frac{1}{\abs{G_0}} \sum_{i \in G_0} X_i} \geq f(2^j; p, \alpha)} \\
               &\geq 1 - \left(j_{\mathsf{hi}} - j_{\mathsf{crit}}\right) \cdot p \\
               &\geq 1 - 2 \log_2\left(\frac{\lub}{\llb}\right) p.
  \end{align*}
  Let us write $\theta_{\ast} := \frac{1}{\abs{G_0}} \sum_{i \in G_0} X_i$.
  Conditioned on the event $\cE$, for any $j, j' \geq j_{\mathsf{crit}}$ we
  have
  \begin{align*}
    \opnorm{{\theta}_{2^j} - {\theta}_{2^{j'}}} &\leq
    \opnorm{{\theta}_{2^j} - \theta_{\ast}}
    + \opnorm{{\theta}_{2^{j'}} - \theta_{\ast}} \leq
    f(2^{j}; p, \alpha) + f(2^{j'}; p, \alpha).
  \end{align*}
  Consequently, it follows that $2^{j_{\mathsf{crit}}}$ satisfies the condition of
  the estimator, and therefore
  \[
    \hat{\lambda} \leq 2^{j_{\mathsf{crit}}} \leq 2 \max\set{\llb, \opnorm{\Sigma_{G_0}}}.
  \]
  Finally, the desired claim follows since
  \begin{align*}
    \opnorm{{\theta}_{\hat{\lambda}} - \theta_{\ast}} &\leq
    \opnorm{{\theta}_{\hat{\lambda}} -
    {\theta}_{2^{j_{\mathsf{crit}}}}} +
    \opnorm{{\theta}_{2^{j_{\mathsf{crit}}}} - \theta_{\ast}} \\
                                                                       &\leq
    f(\hat{\lambda}; p, \alpha) + 2f(2^{j_{\mathsf{crit}}}; p, \alpha) \\
                                                                       &\leq
    3f(2^{j_{\mathsf{crit}}}; p, \alpha) \\
                                                                       &\leq
    171 \sqrt{\opnorm{\Sigma_{G_0}}}
    \left(\alpha + \frac{4\log(1 / p)}{m}\right)^{1/2}.
  \end{align*}
\end{proof}

The next Lemma provides an upper bound on the operator norm of the empirical
covariance $\Sigma_{\igood}$.
\begin{lemma}
  \label{lemma:good-set-covariance}
  Suppose that $\mtrv_{\mathsf{ref}}$ satisfies $\delta_{r}(A) \geq
  8 \, \dist(\mtrv_{\mathsf{ref}}, V)$. Then we have
  \begin{equation}
    \opnorm{\Sigma_{\igood}} \leq
    \biggopnorm{%
      \frac{1}{\abs{\igood}} \sum_{i \in \igood}  V_i V_i^{\T} - VV^{\T}
    } + 2 \, \biggopnorm{%
      \frac{1}{\abs{\igood}} \sum_{i \in \igood} \widetilde{V}_i - V
    }.
    \label{eq:good-set-covariance}
  \end{equation}
\end{lemma}
\begin{proof}
  Let $\mu$ denote the empirical mean over $\igood$. We have
  \begin{align*}
    \mu &= \frac{1}{\abs{\igood}} \sum_{i \in \igood} \widetilde{V}_i, \\
    \Sigma_{\igood} &= \frac{1}{\abs{\igood}}
    \sum_{i \in \igood} (\widetilde{V}_i - \mu)(\widetilde{V}_i - \mu)^{\T} \\
    &=
    \frac{1}{\abs{\igood}} \sum_{i \in \igood} \widetilde{V}_i \widetilde{V}_i^{\T} - \mu \mu^{\T} \\
    &=
    \frac{1}{\abs{\igood}} \sum_{i \in \igood} \widetilde{V}_i \widetilde{V}_i^{\T} - VV^{\T}
    + VV^{\T} - \mu \mu^{\T} \\
    &= \frac{1}{\abs{\igood}} \sum_{i \in \igood} \widetilde{V}_i \widetilde{V}_i ^{\T} - VV^{\T}
    + (V - \mu)(V + \mu)^{\T},
  \end{align*}
  where $V \in \ortho_{d, r}$ spans the principal eigenspace of $\sol$ and satisfies
  \[
    \min_{Z \in \ortho_r} \frobnorm{VZ - \mtrv_{\mathsf{ref}}} =
    \frobnorm{V - \mtrv_{\mathsf{ref}}}.
  \]
  We now bound the spectral norm of $\Sigma_{\igood}$. Indeed, we have
  \begin{align*}
    \opnorm{\Sigma_{\igood}} &\leq
    \biggopnorm{\frac{1}{\abs{\igood}} \sum_{i \in \igood} \widetilde{V}_i \widetilde{V}_i^{\T} - VV^{\T}}
    + \opnorm{V + \mu} \biggopnorm{\frac{1}{\abs{\igood}} \sum_{i \in \igood} \widetilde{V}_i - V} \notag \\
    &\leq
    \biggopnorm{\frac{1}{\abs{\igood}} \sum_{i \in \igood} V_i V_i^{\T} - VV^{\T}}
    + 2 \, \biggopnorm{\frac{1}{\abs{\igood}} \sum_{i \in \igood} \widetilde{V}_i - V},
  \end{align*}
  using the fact that $\widetilde{V}_i \widetilde{V}_i^{\T} = V_i V_i^{\T}$ for all
  $i \in \igood$.
\end{proof}

\subsection{Proof of~\cref{theorem:filter-algorithm-guarantees}}
In this section, we modify the proof of~\cite[Theorem 4]{PBR19} to derive
guarantees for robust mean estimation with matrix-valued inputs. We recall
some notation used therein: given the set of ``good'' samples $G_0$
and the initial sample $S_0 = \set{1, \dots, m}$, we denote
\begin{equation}
  \begin{aligned}
    S_k &= \set{\text{points remaining after $k$ recursive calls to \texttt{Filter}}}, \\
    G_k &= S_k \cap G_0, \\
    B_k &= S_k \setD G_0, \\
    \alpha &= \frac{m - \abs{G_0}}{m}.
  \end{aligned}
  \label{eq:proof-sets}
\end{equation}
Moreover, given any set $S \subset [m]$, we write
\begin{equation}
  \Sigma_{S} := \frac{1}{\abs{S}} \sum_{i \in S} (X_i - \mu_S)(X_i - \mu_S)^{\T},
  \quad \text{where} \quad
  \mu_S := \frac{1}{\abs{S}} \sum_{i \in S} X_i.
  \label{eq:empirical-covariance-over-set}
\end{equation}
In our proofs, we frequently employ the total variation distance $d_{\mathrm{TV}}$.
For discrete distributions $P_1$, $P_2$ on a common sample space $\Omega$,
$d_{\mathrm{TV}}$ is given by
\begin{equation}
  d_{\mathrm{TV}}(P_1, P_2) = \frac{1}{2} \norm{P_1 - P_2}_1
  = \frac{1}{2} \sum_{x \in \Omega} \abs{P_1(x) - P_2(x)}.
  \label{eq:dtv-definition}
\end{equation}
Finally, we define the events $\cE_{k}$, where $k \in \Nbb$, as below:
\begin{equation}
  \cE_{k} := \set{
    \sum_{i \in G_k} \tau_i \geq \frac{1}{\gamma}
    \sum_{j \in S_k} \tau_j
  },
  \quad k = 0, 1, \dots
  \label{eq:bad-events}
\end{equation}
Our proof essentially traces the proof of~\cite[Theorem 4]{PBR19} but for the
case of matrix-valued inputs to the \texttt{Filter} algorithm. The first result
has already been shown in~\cite{PBR19}, as its proof is independent of the shape
of the inputs.
\begin{lemma}[{See~\cite[Lemma 6]{PBR19}}]
  \label{lemma:stopping-time}
  Let $T := \inf\set{k \in \Nbb \mmid \text{$\cE_k$ is true}}$. Then we have:
  \begin{equation}
    \prob{T \geq 3(m - \abs{G_0}) + 18 \log(1 / p)} \leq p.
  \end{equation}
\end{lemma}
The remainder of the proof is devoted to showing that, as soon as some $\cE_k$
is true, \texttt{Filter} will terminate with a good estimate. Throughout, we
condition on the event
\begin{equation}
  \cE := \set{T \leq T_p}, \quad 
  \text{where} \quad T_p := 3(m - \abs{G_0}) + 18 \log(1 / p),
  \label{eq:stopping-time-event}
\end{equation}
which holds with probability at least $1 - p$.
\begin{theorem}
  \label{theorem-algorithm-master}
  Suppose that $\alpha$, $p$ and $N$ satisfy
  \begin{equation}
    3\alpha + \frac{18 \log(1 / p)}{m} \leq \frac{1}{4}.
    \label{eq:corruption-ub}
  \end{equation}
  Then the following hold simultaneously with probability at least $1 - p$:
  \begin{enumerate}
    \item $\mathtt{Filter}(S_0, \opnorm{\Sigma_{G_0}})$ terminates after at most $T_{p}$ iterations;
    \item The output of $\mathtt{Filter}(S_0, \opnorm{\Sigma_{G_0}})$, ${\theta}_{\opnorm{\Sigma_{G_0}}}$, satisfies
      \begin{equation}
        \Bigopnorm{{\theta}_{\opnorm{\Sigma_{G_0}}} - \frac{1}{\abs{G_0}} \sum_{i \in G_0} X_i}
        \leq 18 \sqrt{5 \opnorm{\Sigma_{G_0}}} \left(\alpha + \frac{4\log(1 / p)}{m}\right)^{1/2}.
        \label{eq:filter-master-bound}
      \end{equation}
  \end{enumerate}
\end{theorem}
\begin{remark}{\it
  While we prove the Theorem for the case $\lub = \opnorm{\Sigma_{G_0}}$, a
  straightforward modification of the proof shows that when
  $\lub \geq \opnorm{\Sigma_{G_0}}$, we have
  \[
    \Bigopnorm{\theta_{\lub} - \frac{1}{\abs{G_0}} \sum_{ i \in G_0 } X_i}
    \leq 18 \sqrt{5 \lub} \left(\alpha + \frac{4 \log(1 / p)}{m}\right)^{1/2}.
  \]
}
\end{remark}
\begin{proof}[Proof of~\cref{theorem-algorithm-master}]
  We condition on the event $\cE$ from~\eqref{eq:stopping-time-event}, which
  holds with probability at least $1 - p$. This implies that there is some
  index $k \leq T_p$ such that
  \[
    \sum_{i \in G_k} \tau_i \geq \frac{1}{\gamma} \sum_{j \in S_k} \tau_j.
  \]
  From~\cref{lemma:distance-between-means}, we obtain that the empirical
  covariance satisfies
  \(
    \opnorm{\Sigma_{S_k}} \leq 18 \opnorm{\Sigma_{G_0}},
  \)
  and thus the algorithm terminates after at most $k$ steps. We have
  the following cases:
  \begin{enumerate}
    \item The termination condition was first triggered at the $k^{\text{th}}$
      step. In that case,~\cref{lemma:distance-between-means} directly implies
      the desired inequality.
    \item The algorithm terminated at some index $\ell < k$. Then it follows
      from~\cref{lemma:dtv-P1-P2} that
      \begin{equation}
        \eta := d_{\mathrm{TV}}(\mathrm{Unif}(S_{\ell}), \mathrm{Unif}(G_0)) \leq
        5 \alpha + \frac{20 \log(1 / p)}{N}.
        \label{eq:proof-master-dtv}
      \end{equation}
      At the same time, \cref{lemma:dtv-bound} implies that
      \begin{equation}
        \Bigopnorm{{\theta}_{\opnorm{\Sigma_{G_0}}} - \frac{1}{\abs{G_0}}
        \sum_{i \in G_0} X_i} \leq
        \frac{\sqrt{\eta}}{1 - \sqrt{\eta}} \cdot
        \left(\opnorm{\Sigma_{S_{\ell}}}^{1/2} + \opnorm{\Sigma_{G_0}}^{1/2}\right).
        \label{eq:proof-master-mean-diff}
      \end{equation}
      From the termination condition, we obtain that
      \begin{equation}
        \opnorm{\Sigma_{S_{\ell}}} \leq 18 \opnorm{\Sigma_{G_0}}.
        \label{eq:proof-master-condition}
      \end{equation}
      Combining~\cref{eq:proof-master-dtv,eq:proof-master-mean-diff,eq:proof-master-condition}
      yields the desired bound.
  \end{enumerate}
\end{proof}
The next few Lemmas are supporting statements used in the proof of~\cref{theorem-algorithm-master}.
\begin{lemma}
  \label{lemma:dtv-bound}
  Let $S = \set{X_1, \dots, X_m}$ where $X_i \in \Rbb^{d \times r}$ and suppose
  that $P_1$, $P_2$ are discrete distributions supported over $[m]$ with
  $d_{\mathrm{TV}}(P_1, P_2) = \eta$. Then the following holds:
  \begin{equation}
    \opnorm{\expec[P_1]{X_i} - \expec[P_2]{X_i}} \leq
    \frac{\sqrt{\eta}}{1 - \sqrt{\eta}} \cdot \left(
      \opnorm{\Sigma_{P_1}}^{1/2} + \opnorm{\Sigma_{P_2}}^{1/2}
    \right),
    \label{eq:dtv-bound}
  \end{equation}
  where the matrices $\Sigma_{P_i}$ are defined as:
  \[
    \Sigma_{P_i} = \expec[X \sim P_i]{(X - \expec[P_i]{X})(X - \expec[P_i]{X})^{\T}}.
  \]
\end{lemma}
\begin{proof}
  Following the proof of~\cite[Lemma 2.1]{KS17}, we consider a coupling between
  $P_1$ and $P_2$ such that $\prob{X = X'} \ge 1 - \eta$. Denoting
  $\norm{X}_{L^2} := \sqrt{\expec{X^2}}$, we have
  \begin{align}
    \opnorm{\expec[P_1]{X} - \expec[P_2]{X'}} &=
    \sup_{u, v \in \cB} \ip{u, (\expec[P_1]{X} - \expec[P_2]{X'})v} \notag \\
                                              &=
    \sup_{u, v \in \cB} \expec{\ip{u, (X - X') v} \bm{1}\set{X \neq X'}} \notag \\
                                              &\leq
    \expec{\bm{1}\set{X \neq X'}^{2}}^{1/2} \cdot
    \sup_{u, v \in \cB} \expec{\ip{u, (X - X') v}^2}^{1/2} \notag \\
                                              &\leq
    \sqrt{\eta} \cdot \sup_{u, v \in \cB} \norm{\ip{u, (X - X') v}}_{L^2}.
    \label{eq:dtv-master-bound}
  \end{align}
  Let $\mu_1 := \expec[P_1]{X}$ and $\mu_2 = \expec[P_2]{X}$. Since $\norm{\cdot}_{L^2}$
  is a norm, the triangle inequality implies that
  \begin{align}
    \sup_{u, v \in \cB} \norm{\ip{u, (X - X') v}}_{L^2}
      &=
    \sup_{u, v \in \cB} \norm{\ip{u, (X - \mu_1 + \mu_1 - \mu_2 + \mu_2 - X') v}}_{L^2}
    \notag \\
      &\leq
      \begin{aligned}[t]
        & \sup_{u, v \in \cB} \norm{\ip{u, (X - \mu_1) v}}_{L^2}
        + \sup_{u, v \in \cB} \norm{\ip{u, (X' - \mu_2) v}}_{L^2} \\
        & + \sup_{u, v \in \cB} \norm{\ip{u, (\mu_1 - \mu_2) v}}_{L^2}.
      \end{aligned}
    \label{eq:dtv-bound-decomp}
  \end{align}
  We now upper bound the remaining terms. For the first one, we have
  \begin{align}
    \sup_{u, v \in \cB} \norm{\ip{u, (X - \mu_1) v}}_{L^2} &=
    \sup_{u, v \in \cB} \expec{\ip{u, (X - \mu_1) v}^2}^{1/2} \notag \\
                                    &=
    \sup_{u, v \in \cB} \Ebb\bigl[\mathrm{Tr}\bigl(u^{\T} (X - \mu_1)
      \underbrace{vv^{\T}}_{\preceq I_d}
    (X - \mu_1)^{\T} u\bigr)\bigr]^{1/2} \notag \\
                                      &\leq
    \sup_{u \in \cB} \expec{\trace{u^{\T}(X - \mu_1)(X - \mu_1)^{\T} u}}^{1/2} \notag \\
                                      &=
    \left(\sup_{u \in \cB}\ip{u, \expec{(X - \mu_1)(X - \mu_1)^{\T}} u}\right)^{1/2}
    \notag \\
                                      &=
    \opnorm{\Sigma_{P_1}}^{1/2},
    \label{eq:dtv-bound-1}
  \end{align}
  where the penultimate equality uses linearity of the trace operator and the
  last equality is the definition of the spectral norm for symmetric positive
  semidefinite matrices. Similar arguments also yield
  \begin{align}
    \sup_{u, v \in \cB} \norm{\ip{u, (X' - \mu_2) v}}_{L^2} &\leq
    \opnorm{\Sigma_{P_2}}^{1/2},
    \label{eq:dtv-bound-2} \\
    \sup_{u, v \in \cB} \norm{\ip{u, (\mu_1 - \mu_2) v}}_{L^2} &\leq
    \opnorm{\expec[P_1]{X} - \expec[P_2]{X'}}.
    \label{eq:dtv-bound-3}
  \end{align}
  Plugging~\cref{eq:dtv-bound-decomp,eq:dtv-bound-1,eq:dtv-bound-2,eq:dtv-bound-3}
  back into \cref{eq:dtv-master-bound} and rearranging yields the expected result:
  \[
    \opnorm{\expec[P_1]{X} - \expec[P_2]{X'}} \leq
    \frac{\sqrt{\eta}}{1 - \sqrt{\eta}} \left(
      \opnorm{\Sigma_{P_1}}^{1/2} + \opnorm{\Sigma_{P_2}}^{1/2}
    \right).
  \]
\end{proof}

\begin{lemma}
  \label{lemma:scores-on-good-set-larger}
  Let $G \subset S \subset [m]$. Moreover, let $\mu_S$ and
  $\mu_G$ be their respective empirical means, and let $v$ be the leading
  eigenvector of $\Sigma_S$ so that the outlier scores satisfy
  \[
    \tau_i = \ip{v, (X_i - \mu_S)(X_i - \mu_S)^{\T} v}, \quad
    \forall i \in S.
  \]
  Moreover, define $\eta := 1 - \nicefrac{\abs{G}}{\abs{S}}$ and fix a
  $\gamma \in (0, \nicefrac{1}{\eta})$. Then, we have the implication
  \begin{equation}
    \opnorm{\Sigma_{S}} \geq (1 - \eta)^2 \left(\frac{\gamma}{1 - \gamma \eta}\right) \opnorm{\Sigma_{G}}
    \implies \sum_{j \in G} \tau_j \leq \frac{1}{\gamma} \sum_{i \in S} \tau_i.
    \label{eq:scores-on-good-set-larger}
  \end{equation}
\end{lemma}
\begin{proof}
  Recall that the (normalized) sum of outlier scores over the set $G$ is given by
  \begin{align}
    \frac{1}{\abs{G}} \ip{v, \sum_{i \in G} (X_i - \mu_S)(X_i - \mu_S)^{\T} v} &=
    \begin{aligned}[t]
    & \frac{1}{\abs{G}} \ip{v, \sum_{i \in G} (X_i - \mu_G)(X_i - \mu_G)^{\T} v} \\
    & + \ip{v, (\mu_S - \mu_G)(\mu_S - \mu_G)^{\T} v}
    \end{aligned} \notag \\
                                                                               &=
    \ip{v, \Sigma_G v} + \ip{v, (\mu_S - \mu_G)(\mu_S - \mu_G)^{\T} v}.
    \label{eq:outlier-score-sum-1}
  \end{align}
  We now simplify the second term. Indeed, we have
  \begin{align}
    \mu_S - \mu_G &= \frac{1}{\abs{S}} \sum_{i \in G} X_i +
    \frac{1}{\abs{S}} \sum_{i \in S \setD G} X_i - \frac{1}{\abs{G}} \sum_{i \in G} X_i \notag \\
                  &=
                  \left(1 - \frac{\abs{G}}{\abs{S}}\right) \left(\mu_{S \setD G} - \mu_{G}\right)
    \label{eq:mean-diff-1}
  \end{align}
  For brevity, denote $\eta := \frac{\abs{S \setD G}}{\abs{S}}$.
  Plugging~\eqref{eq:mean-diff-1} back into~\eqref{eq:outlier-score-sum-1}, we obtain
  \begin{align}
    \frac{1}{\abs{G}} \sum_{j \in G} \tau_j &=
    \ip{v, \Sigma_G v} + \eta^2
    \ip{v, (\mu_{S \setD G} - \mu_{G})(\mu_{S \setD G} - \mu_{G})^{\T} v}
    \label{eq:outlier-score-sum-2}
  \end{align}
  We now bound the second term in~\eqref{eq:outlier-score-sum-2}.
  From~\cite[Lemma 2.4]{DK19}, it follows that
  \begin{align*}
    \ip{v, \Sigma_S v} &= \left(1 - \eta\right) \ip{v, \Sigma_G v}
    + \eta \ip{v, \Sigma_{S \setD G} v} + \eta(1 - \eta)
    \ip{v, (\mu_{S \setD G} - \mu_G)(\mu_{S \setD G} - \mu_G)^{\T} v}
  \end{align*}
  Rearranging and multiplying by $\eta/ (1 - \eta)$ gives
  \begin{align*}
    \eta^2 \ip{v, (\mu_{S\setD G} - \mu_G)(\mu_{S\setD G} - \mu_G)^{\T} v} &=
    \frac{\eta}{1 - \eta} \ip{v, \Sigma_S v} -
    \eta \ip{v, \Sigma_{G} v} - \frac{\eta^2}{(1 - \eta)} \ip{v, \Sigma_{S \setD G} v} \\
                                                                           &\leq
    \frac{\eta}{1 - \eta} \ip{v, \Sigma_S v} - \eta \ip{v, \Sigma_{G} v}.
  \end{align*}
  Plugging back into~\cref{eq:outlier-score-sum-2} and using the fact that
  $\abs{G} = \abs{S}(1 - \eta)$, we obtain
  \begin{align}
    \sum_{j \in G} \tau_j &\leq
    \abs{G} (1 - \eta)\ip{v, \Sigma_G v} + \frac{\abs{G} \eta}{1 - \eta} \ip{v, \Sigma_S v} \notag \\
                          &\leq
                          \abs{G} (1 - \eta) \opnorm{\Sigma_{G}} + (\abs{S} - \abs{G}) \abs{S} \opnorm{\Sigma_{S}}
    \label{eq:outlier-score-sum-3}
  \end{align}
  Finally, replacing $\abs{G} = \abs{S}(1 - \eta)$
  in~\eqref{eq:outlier-score-sum-3} and rearranging, we obtain
  \[
    \opnorm{\Sigma_{G}} \leq
    \left(\gamma^{-1} - \eta\right) \frac{\opnorm{\Sigma_{S}}}{(1 - \eta)^2} \implies
    \sum_{j \in G} \tau_j \leq \frac{1}{\gamma} \sum_{i \in S} \tau_i.
  \]
\end{proof}

\begin{lemma}
  \label{lemma:bad-proportion} 
  Suppose that~\eqref{eq:corruption-ub} is true. Then the
  following holds for any $k \le T_{p}$:
  \begin{align*}
    \frac{\abs{S_k \setD G_k}}{\abs{S_k}} \leq \frac{4\alpha}{3}.
  \end{align*}
\end{lemma}
\begin{proof}
  Recall that $B_k = S_k \setD G_k$ and notice that
  \begin{align*}
    \frac{\abs{B_k}}{\abs{S_k}}
      &= \frac{\abs{B_k}}{\abs{S_0}} \frac{\abs{S_0}}{\abs{S_k}}
    \leq
      \frac{\abs{B_0}}{\abs{S_0}} \frac{\abs{S_0}}{\abs{S_0} - T_{p}}
      = \alpha \cdot \frac{1}{1 - (3\alpha + \frac{18\log(1 / p)}{m})} \leq
      \frac{4\alpha}{3},
  \end{align*}
  where the first inequality follows from the fact that $\abs{B_k} \leq
  \abs{B_0}$.
\end{proof}

\begin{lemma}
  \label{lemma:covariance-matrix-inequality}
  For any integer $k$, the sets $G_k$ and $G_0$ satisfy
  \begin{equation}
    \opnorm{\Sigma_{G_k}} \leq \frac{\abs{G_0}}{\abs{G_k}} \opnorm{\Sigma_{G_0}}
    \label{eq:covariance-matrix-inequality}
  \end{equation}
\end{lemma}
\begin{proof}
  We expand the definition of $\Sigma_{G_0}$ and rewrite:
  \begin{align*}
    \Sigma_{G_0} &= \frac{1}{\abs{G_0}}
      \sum_{i \in G_0} (X_i - \mu_{G_0})(X_i - \mu_{G_0})^{\T} \\
                 &= \underbrace{\frac{1}{\abs{G_0}}
                 \sum_{i \in G_k} (X_i - \mu_{G_0})(X_i - \mu_{G_0})^{\T}}_{T_1} +
    \underbrace{\frac{1}{\abs{G_0}} \sum_{i \in G_0 \setD G_k} (X_i - \mu_{G_0})(X_i - \mu_{G_0})^{\T}}_{T_2}
  \end{align*}
  We now rewrite the first term in the above sum using
  \begin{align*}
    T_1 &= \frac{1}{\abs{G_0}} \sum_{i \in G_k} (X_i - \mu_{G_k} + \mu_{G_k} - \mu_{G_0})(X_i - \mu_{G_k} + \mu_{G_k} - \mu_{G_0})^{\T} \\
        &= \begin{aligned}[t]
        & \frac{\abs{G_k}}{\abs{G_0}} \Sigma_{G_k}
        + \frac{\abs{G_k}}{\abs{G_0}} \left(\frac{1}{\abs{G_k}} \sum_{i \in G_k} (X_i - \mu_{G_k})\right) (\mu_{G_k} - \mu_{G_0})^{\T} \\
        & + \frac{\abs{G_k}}{\abs{G_0}}
        (\mu_{G_k} - \mu_{G_0}) \left(\frac{1}{\abs{G_k}} \sum_{i \in G_k} X_i - \mu_{G_k}\right)^{\T}
        + \frac{\abs{G_k}}{\abs{G_0}} (\mu_{G_k} - \mu_{G_0})(\mu_{G_k} - \mu_{G_0})^{\T}
      \end{aligned} \\
        &=
        \frac{\abs{G_k}}{\abs{G_0}} \left(
          \Sigma_{G_k} + (\mu_{G_k} - \mu_{G_0})(\mu_{G_k} - \mu_{G_0})^{\T}
        \right)
  \end{align*}
  Letting $v \in \Sbb^{d-1}$ and using the fact that $T_2$ is
  positive semidefinite, we arrive at
  \begin{equation}
    \ip{v, \Sigma_{G_0} v} = \frac{\abs{G_k}}{\abs{G_0}}
    \left( \ip{v, \Sigma_{G_k} v} + \norm{(\mu_{G_k} - \mu_{G_0})^{\T} v}^2 \right)
    + \ip{v, T_{2} v} \geq \frac{\abs{G_k}}{\abs{G_0}} \ip{v, \Sigma_{G_k} v}
  \end{equation}
  Finally, taking suprema over both sides yields the desired inequality.
\end{proof}

\begin{lemma}
  \label{lemma:distance-between-means}
  Suppose that~\eqref{eq:corruption-ub} is true
  and that the following inequality holds for some index $k \leq T_p$:
  \begin{equation}
    \sum_{i \in G_k} \tau_i \geq \frac{1}{\gamma}
    \sum_{j \in S_k} \tau_j.
    \label{eq:outlier-scores-contrapositive}
  \end{equation}
  Then the empirical means satisfy
  \[
    \opnorm{\expec[\mathrm{Unif}(G_0)]{X} - \expec[\mathrm{Unif}(S_k)]{X}}
    \leq 18 \left(
      5 \alpha + \frac{20 \log(1 / p)}{m}
    \right)^{1/2} \opnorm{\Sigma_{G_0}}^{1/2}.
  \]
\end{lemma}
\begin{proof}
  Let $P_1 := \mathrm{Unif}(G_0)$ and $P_2 := \mathrm{Unif}(S_k)$.
  From~\cref{lemma:dtv-bound}, it follows that
  \begin{equation}
    \opnorm{\expec[P_1]{X} - \expec[P_2]{X}} \leq
    \frac{\sqrt{d_{\mathrm{TV}}(P_1, P_2)}}{1 - \sqrt{d_{\mathrm{TV}}(P_1, P_2)}}
    \cdot \left(
      \opnorm{\Sigma_{G_0}}^{1/2} + \opnorm{\Sigma_{G_k}}^{1/2}
    \right).
    \label{eq:distance-between-means-main}
  \end{equation}
  Since~\eqref{eq:outlier-scores-contrapositive} is the reverse of~\eqref{eq:scores-on-good-set-larger},
  we obtain
  \begin{align*}
    \opnorm{\Sigma_{S_k}} &\leq (1 - \eta)^2 \frac{\gamma}{1 - \gamma \eta}
    \opnorm{\Sigma_{G_k}} \\
                          &\leq \frac{3}{1 - 6\alpha} \opnorm{\Sigma_{G_k}} \\
                          &\leq
    6 \cdot \frac{\abs{G_0}}{\abs{G_k}} \opnorm{\Sigma_{G_0}},
  \end{align*}
  where the first inequality follows from the contrapositive of~\cref{lemma:scores-on-good-set-larger},
  the second inequality from $\gamma = 3$ and~\cref{lemma:bad-proportion}, and
  the last inequality follows by our assumption on $\alpha$.
  Now, let $K \leq T_{p}$ be the number of samples in $G_0$ that were removed
  by the algorithm. We have
  \begin{align*}
    \frac{\abs{G_0}}{\abs{G_k}} &=
    \frac{m - \abs{B_0}}{m - \abs{B_0} - K} \leq
    \frac{m - \abs{B_0}}{m - \abs{B_0} - T_{p}} \leq
    \frac{m - \abs{B_0}}{m - 18 \log(1 / p) - 4 \abs{B_0}} =
    \frac{1 - \alpha}{1 - 4 \alpha - \frac{18 \log(1 / p)}{m}}
  \end{align*}
  From~\eqref{eq:corruption-ub}, we additionally have that
  \[
    1 - (4 \alpha + \frac{18 \log(1 / p)}{m}) \geq
    1 - \frac{4}{3} \left(3 \alpha + \frac{18 \log(1 / p)}{m}\right) \geq
    \frac{1}{3} \implies
    \opnorm{\Sigma_{S_k}} \leq 18 \opnorm{\Sigma_{G_0}}.
  \]
  Substituting the above into~\eqref{eq:distance-between-means-main} and using
  \cref{lemma:dtv-P1-P2} yields the desired bound:
  \begin{align*}
    \opnorm{\expec[P_1]{X} - \expec[P_2]{X}} &\leq
    \frac{\left(5 \alpha + \frac{20 \log(1 / p)}{m}\right)^{1/2}}{%
    1 - \left(5 \alpha + \frac{20 \log(1 / p)}{m}\right)^{1/2}}
    \left(\opnorm{\Sigma_{G_0}}^{1/2} + \sqrt{18} \opnorm{\Sigma_{G_0}}^{1/2}\right) \\
                                             &\leq
    18 \left(5 \alpha + \frac{20 \log(1 / p)}{m}\right)^{1/2} \opnorm{\Sigma_{G_0}}^{1/2}.
  \end{align*}
\end{proof}

\begin{lemma}
  \label{lemma:dtv-P1-P2}
  Suppose $k \leq T_p$ and~\eqref{eq:corruption-ub} holds. Then we have that
  \begin{equation}
    d_{\mathrm{TV}}(\mathrm{Unif}(S_k), \mathrm{Unif}(G_0)) \leq 5\alpha + \frac{20\log(1 / p)}{m}.
    \label{eq:dtv-P1-P2}
  \end{equation}
\end{lemma}
\begin{proof}
  We let $P_1 := \mathrm{Unif}(S_k)$, $P_2 := \mathrm{Unif}(G_0)$ and $P_3 := \mathrm{Unif}(G_k)$,
  and write $K \leq k \leq T_p$ for the number of samples originally in $G_0$ that were removed
  by the \texttt{Filter} algorithm by the $k^{\text{th}}$ step.
  From the triangle inequality, it follows that
  \begin{align*}
    d_{\mathrm{TV}}(P_1, P_2) &\leq
    d_{\mathrm{TV}}(P_1, P_3) + d_{\mathrm{TV}}(P_2, P_3) \\
                              &=
    \frac{\abs{S_k} - \abs{G_k}}{\abs{S_k}}
    + \frac{\abs{G_0} - \abs{G_k}}{\abs{G_0}} \\
                              &=
    \frac{m - k - (m - \abs{B_0} - K)}{m - k}
    + \frac{K}{m - \abs{B_0}} \\
                              &=
    \frac{\abs{B_0} + (K - k)}{m - k} +
    \frac{K}{m - \abs{B_0}} \\
                              &\leq
    \frac{\abs{B_0}}{m - k} +
    \frac{T_{p}}{m - \abs{B_0}} \\
                              &\leq
    \frac{\abs{B_0}}{m - T_p}
    + \frac{T_p}{m - \abs{B_0}},
  \end{align*}
  where the second line follows from~\cref{lemma:dtv-uniform} and
  the last two inequalities follow from $K \leq m$ and $m \leq T_p$. Finally,
  using~\cref{lemma:stopping-time} and~\cref{eq:corruption-ub}, we obtain
  \begin{align*}
    \frac{\abs{B_0}}{m - T_p} +
    \frac{T_p}{m - \abs{B_0}} &=
    \frac{\alpha}{1 - \frac{T_p}{m}} +
    \frac{\frac{T_p}{m}}{1 - \alpha} \\
                              &\leq
    \frac{\alpha}{1 - \frac{18 \log(1 / p)}{m} - 3\alpha} +
    \frac{\frac{18 \log(1 / p)}{m} + 3\alpha}{1 - \alpha} \\
                              &\leq
    \frac{4\alpha}{3} + \frac{\frac{18 \log(1 / p)}{m} + 3\alpha}{1 - \alpha} \\
                              &\leq
    \frac{4\alpha}{3} + \frac{\frac{18 \log(1 / p)}{m} + 3 \alpha}{1 - \frac{1}{12}} \\
                              &\leq
    5 \alpha + \frac{20 \log(1 / p)}{m}.
  \end{align*}
\end{proof}

\begin{lemma} \label{lemma:dtv-uniform}
  Consider a pair of discrete sets $S, S'$ such that $S' \subset S$. We have:
  \begin{equation}
    d_{\mathrm{TV}}(\mathrm{Unif}(S), \mathrm{Unif}(S')) =
    \frac{\abs{S} - \abs{S'}}{\abs{S}}.
    \label{eq:dtv-uniform}
  \end{equation}
\end{lemma}
\begin{proof}
  Using the fact that $d_{\mathrm{TV}}(p, q) = \frac{1}{2} \norm{p - q}_{1}$, we have:
  \begin{align*}
    d_{\mathrm{TV}}(\mathrm{Unif}(S), \mathrm{Unif}(S')) &= \frac{1}{2}
    \left(
    \sum_{x \in S \cap S'} \abs{\frac{1}{\abs{S}} - \frac{1}{\abs{S'}}}
    + \sum_{x \in S \setD S'} \frac{1}{\abs{S}} \right) \\
                                                         &=
                                                         \frac{1}{2} \left( 1 - \frac{\abs{S'}}{\abs{S}} +
    \frac{\abs{S} - \abs{S'}}{\abs{S}} \right) \\
                                                         &=
                                                         1 - \frac{\abs{S'}}{\abs{S}}.
  \end{align*}
\end{proof}

\subsection{Proof of Theorem~\ref{theorem:distributed-pca-rate}}
\label{sec:appendix:theorem:distributed-pca-rate}
We now present the proof of the main theorem on distributed PCA.
We first recall that
\[
  \est_i = \frac{1}{n} \sum_{j=1}^n X_j^{(i)} (X_j^{(i)})^{\T}; \quad
  i \in \igood,
\]
where $X_j^{(i)} \iid \cP$, and that the responses $V_i \in \ortho_{d, r}$ span
the leading $r$-dimensional eigenspace of $\est_i$.
Under this model, the local errors $E_i := \est_i - \sol$ as well as the error
of the empirical average over the inliers are bounded with high probability. We
will condition on the following events for the remainder of this section:
\begin{equation}
  \label{eq:conditioned-events}
  \begin{aligned}
    \cE_1 &= \set{\max_{i \in \igood} \opnorm{\est_i - \sol}
      \leq \min\set{\frac{\delta}{8}, C_1 \opnorm{\sol}
        \sqrt{\frac{r_{\star} + \log(m / p)}{n}}}}, \\
    \cE_2 &= \set{
      \biggopnorm{\frac{1}{\abs{\igood}} \sum_{i \in \igood} \est_i - \sol}
      \leq C_2 \opnorm{\sol} \sqrt{\frac{r_{\star} + \log(n)}{\abs{\igood} n}}}.
  \end{aligned}
\end{equation}
\begin{lemma}
  \label{lemma:matrix-errors-high-probability}
  Suppose that $n \gtrsim \kappa^2 \cdot (r_{\star} + \log(mn / p))$. Then the
  following hold:
  \begin{equation}
    \prob{\cE_1}
      \leq p, \quad
    \prob{\cE_2}
      \leq \frac{2}{n}.
    \label{eq:aggregate-error}
  \end{equation}
\end{lemma}
\begin{proof}
  The bound on $\prob{\cE_2}$ in~\cref{eq:aggregate-error}
  follows from an application of~\cite[Exercise 9.2.5]{Vershynin18} and the
  assumed lower bound on $n$. On the other hand, the same result yields
  \[
    \prob{\opnorm{\est_i - \sol} \geq
      C_1 \opnorm{A} \left(\sqrt{\frac{r_{\star} + \log(m / p)}{n}}
        + \frac{r_{\star} + \log(m / p)}{n}
      \right)} \leq \frac{p}{m},
  \]
  for any fixed $i \in \igood$. From the lower bound on $n$, it follows
  that
  \[
    \frac{r_{\star} + \log(m / p)}{n} \leq
    \sqrt{\frac{r_{\star} + \log(m / p)}{n}} \quad \text{and} \quad
    C_1 \opnorm{\sol}
    \sqrt{\frac{r_{\star} + \log(m / p)}{n}} < \frac{\delta}{8}.
  \]
  Finally, taking a union bound over $\igood$ recovers the bound on $\prob{\cE_1}$.
\end{proof}

An immediate corollary is a bound on the error of \texttt{RobustReferenceEstimator}.
\begin{corollary}
  \label{corollary:robust-ref-estimator-pca}
  There is a universal constant $C_{\mathsf{ref}}$ such that the output of 
  Alg.~2
  satisfies
  \[
    \dist(\mtrv_{\mathsf{ref}}, V) \leq
    C_{\mathsf{ref}} \kappa \cdot \sqrt{\frac{r_{\star} + \log(m / p)}{n}}.
  \]
\end{corollary}
\begin{proof}
  From the bound $\alpha < \frac{1}{2}$ and the conditioning on $\cE_1$, we deduce the
  existence of an index set $S'$ such that $\abs{S'} > \frac{m}{2}$, and
  \[
    \dist(\mtrv_{i}, V) \leq
    \frac{\opnorm{\est_i - \sol}}{\delta - \frac{\delta}{4}} \leq
    \frac{2C_1 \opnorm{A}}{\delta}\sqrt{\frac{r_{\star} + \log(m / p)}{n}},
    \quad \text{for all $i \in S'$},
  \]
  where the first bound on $\dist(\mtrv_{i}, V)$ follows from the Davis-Kahan
  theorem~\cite[Theorem 2.7]{CCFM21} and the fact that $\opnorm{\est_i - \sol}
  \leq \frac{\delta}{8}$ for any $i \notin \cI_{\mathrm{bad}}$.
  From Proposition 1 in the main text, it follows that
  \[
    \dist(\mtrv_{\mathsf{ref}}, V) \leq
    \underbrace{6 C_1}_{C_{\mathsf{ref}}} \frac{\opnorm{A}}{\delta} \sqrt{%
      \frac{r_{\star} + \log(m / p)}{n}
    }.
  \]
\end{proof}
The next Proposition instantiates the bounds of~\cref{lemma:good-set-covariance} for
for the case of distributed PCA.
\begin{proposition} \label{proposition:covariance-bound}
  In the setting of~\cref{lemma:good-set-covariance}, the matrix
  $\Sigma_{\igood}$ satisfies
  \begin{align}
    \label{eq:empirical-covariance-opnorm-pca}
    \opnorm{\Sigma_{\igood}} &\lesssim
      \kappa \, \sqrt{\frac{r(r_{\star} + \log(n))}{(1 - \alpha) mn}}
      + \kappa^2 \cdot
      \frac{\sqrt{r} (r_{\star} + \log(n))}{n}
      + \kappa^4 \cdot
      \frac{r_{\star} + \log (m / p)}{n}.
  \end{align}
\end{proposition}
\begin{proof}
  From~\cref{lemma:good-set-covariance}, it follows that
  \begin{align*}
    \opnorm{\Sigma_{\igood}} &\leq \biggopnorm{\frac{1}{\abs{\igood}} \sum_{i \in \igood} V_i V_i^{\T} - VV^{\T}} +
    2 \cdot \biggopnorm{\frac{1}{\abs{\igood}} \sum_{i \in \igood} \widetilde{V}_i - V}
  \end{align*}
  From Proposition 2 in the main text and conditioning $\cE_1$ and $\cE_2$,
  we have
  \begin{align}
    \biggopnorm{\frac{1}{\abs{\igood}} \sum_{i \in \igood} \widetilde{V}_i - V} &\lesssim
    \begin{aligned}[t]
      & \frac{1}{\delta} \biggopnorm{\frac{1}{\abs{\igood}} \sum_{i \in \igood} \est_i - \sol} \\
      & + \left( \frac{\opnorm{\sol}}{\delta} \right)^2 \max\left(
        C_1^2, C_{\mathsf{ref}}^2 \left(\frac{\opnorm{\sol}}{\delta}\right)^2
      \right) \cdot \frac{r_{\star} + \log(m / p)}{n}
    \end{aligned} \notag \\
                                                                    &\leq
    C_2 \kappa \sqrt{\frac{r_{\star} + \log(n)}{(1 - \alpha) mn}} +
    \kappa^2 \max\left(
      C_1^2, C_{\mathsf{ref}}^2 \kappa^2
    \right) \cdot \frac{r_{\star} + \log(m/p)}{n} \notag \\
                                                                    &\lesssim
      \kappa \sqrt{\frac{r_{\star} + \log(n)}{(1 - \alpha) mn}}
      + \kappa^4 \cdot \frac{r_{\star} + \log(m/p)}{n}.
    \label{eq:average-error-1}
  \end{align}
  On the other hand, using~\cite[Theorem 2]{FWWZ19}, we have that
  \[
    \biggopnorm{\frac{1}{\igood} \sum_{i \in \igood} V_i V_i^{\T} - VV^{\T}} \lesssim
    \kappa \sqrt{\frac{r (r_{\star} + \log(n))}{(1 - \alpha) m n}}
    + \kappa^2 \frac{\sqrt{r} (r_{\star} + \log(n))}{n}.
  \]
  Putting all the bounds together yields~\eqref{eq:empirical-covariance-opnorm-pca}.
\end{proof}
We now invoke Proposition 3 and recall that $\varrho$ is defined as
\begin{align}
  \varrho &:=
  \kappa \sqrt{\frac{r(r_{\star} + \log(n))}{(1 - \alpha)mn}} 
  + {\kappa}^2 \cdot \frac{\sqrt{r}(r_{\star} + \log(n))}{n}
  + {\kappa}^4 \cdot \frac{r_{\star} + \log(m / p)}{n}
\end{align}
From that and~\cref{proposition:covariance-bound}, it follows that
Alg. 5 from the main text invoked with $\llb = \omega := \sqrt{\nicefrac{1}{mn}}$
and $\lub = 6$ outputs an estimate satisfying
\begin{align}
  \Bigopnorm{\bar{V} - \frac{1}{\abs{\igood}} \sum_{i \in \igood}
  \widetilde{V}_i} &\lesssim
    \sqrt{\max\set{\varrho, \omega}} \cdot \left(\alpha + \frac{\log(1 / p)}{m}\right)^{1/2} \\
    &=
    \sqrt{\varrho} \cdot \left(\alpha + \frac{\log(1 / p)}{m}\right)^{1/2}
  \label{eq:average-error-2}
\end{align}
with failure probability at most $2 \log_2(6 / \omega) p$.
Finally, from~\cref{eq:average-error-1,eq:average-error-2} it follows that
\begin{align*}
  \opnorm{\bar{V} - V} &\leq
  \Bigopnorm{\bar{V} - \frac{1}{\abs{\igood}} \sum_{i \in \igood} \widetilde{V}_i} +
  \Bigopnorm{\frac{1}{\abs{\igood}} \sum_{i \in \igood} \widetilde{V}_i - V} \\
                         &\lesssim
  \sqrt{\varrho} \left(\alpha + \frac{\log(1 / p)}{m}\right)^{1/2}
  + \kappa \sqrt{\frac{r(r_{\star} + \log(n))}{(1 - \alpha) mn}}
  + \kappa^4 \frac{\sqrt{r}(r_{\star} + \log(m / p))}{n}.
\end{align*}
In particular, the success probability is at least (given that $\omega$ is
set as $\sqrt{\nicefrac{1}{mn}}$):
\[
  1 - p - \frac{2}{n} - 2 \log_2 \left(
    \frac{6}{\omega}
  \right) p \geq 1 - \frac{2}{n} - 2\log_2(6mn)p.
\]


\end{document}